\documentclass[11pt]{article}  

\usepackage{algorithm}
\usepackage{algpseudocode}
\algtext*{EndWhile}
\algtext*{EndIf}
\algtext*{EndFor}
\usepackage{paralist}
\usepackage{enumitem}
\usepackage{bm}
\usepackage{amsmath}
\usepackage{amsthm}
\usepackage{amssymb}
\usepackage{wrapfig}
\usepackage{array}
\usepackage{color}
\usepackage{graphics}
\usepackage{natbib}

\usepackage{tikz}
\usetikzlibrary{arrows,arrows.meta}
\usetikzlibrary{decorations.markings}
\usetikzlibrary{decorations.pathmorphing}
\usetikzlibrary{positioning,shapes,backgrounds}
\usetikzlibrary{calc,intersections,patterns,fit,automata,snakes}
\usepackage{tikzsymbols}

\newcommand{\remove}[1]{}

\definecolor{skipcolor}{rgb}{1, .92, .92}
\definecolor{rejoincolor}{rgb}{.9, 1, 0.9}
\definecolor{frameskip}{rgb}{0.95, 0, 0}
\definecolor{framerejoin}{rgb}{0, .75, 0}

\definecolor{gold}{rgb}{1,0.553,0}
\definecolor{lightbrown}{rgb}{0.9305,0.86275,0.8}
\definecolor{fillcopper}{rgb}{0.722,0.451,0.2}
\definecolor{fillsilver}{rgb}{0.753,0.753,0.753}
\definecolor{fillgray}{rgb}{0.4,0.4,0.4}
\definecolor{filllightgray}{rgb}{0.6,0.6,0.6}
\definecolor{fillLightgray}{rgb}{0.7,0.7,0.7}
\definecolor{fillLLightgray}{rgb}{0.8,0.8,0.8}
\definecolor{fillLLLightgray}{rgb}{0.9,0.9,0.9}
\definecolor{fillred}{rgb}{1,0.15,0.15}
\definecolor{filllightred}{rgb}{1,0.3,0.3}
\definecolor{fillblue}{rgb}{0.2,0.35,1}
\definecolor{filllightblue}{rgb}{0.75,0.85,1}
\definecolor{fillgreen}{rgb}{0.2,0.7,0.2}
\definecolor{filllightgreen}{rgb}{0.65,0.95,0.65}
\definecolor{fillgreenbox}{rgb}{0.65,0.95,0.65}
\definecolor{fillpurple}{rgb}{0.7,0.4,0.74}
\definecolor{redRPI}{rgb}{0.839,0,0.1098}

\def\math#1{$#1$}

\def\mand#1{$$#1$$}

\def\mld#1{\begin{equation}
#1
\end{equation}}

\def\eqan#1{\begin{eqnarray*}
#1
\end{eqnarray*}}

\def\eqar#1{\begin{eqnarray}
#1
\end{eqnarray}}

\def\frac#1#2{{#1\over #2}}

\DeclareSymbolFont{AMSb}{U}{msb}{m}{n}

\def\choose#1#2{
\mathchoice
{\Big({{#1}\atop{#2}}\Big)}
{\big({{#1}\atop{#2}}\big)}
{({{#1}\atop{#2}})}
{({{#1}\atop{#2}})}
}

\def\cl#1{{\cal #1}}

\newcommand{\given}{\mid}

\def\argmin{\mathop{\hbox{argmin}}\limits}

\def\norm#1{{\|#1\|}}

\def\r#1{{\eqref{#1}}}

\newtheorem{theorem}{\bf Theorem}[section]
\newtheorem{lemma}[theorem]{Lemma}

\newcounter{rmnum}
\def\RN#1{\setcounter{rmnum}{#1}\uppercase\expandafter{\romannumeral\value{rmnum}}}
\def\rn#1{\setcounter{rmnum}{#1}\expandafter{\romannumeral\value{rmnum}}}

\remove{

}

\definecolor{shadecolor}{gray}{.85}%
\definecolor{tintedcolor}{gray}{.8}%

{\makeatletter\gdef\reallynopagebreak{\par\nopagebreak\@nobreaktrue}}

\providecommand\remove[1]{}

\DeclareSymbolFont{extraup}{U}{zavm}{m}{n}
\DeclareMathSymbol{\varheart}{\mathalpha}{extraup}{86}
\DeclareMathSymbol{\vardiamond}{\mathalpha}{extraup}{87}

\newcommand{\dieface}[2][none]{%
\begin{tikzpicture}[baseline=-5pt,line width=2pt]
\begin{scope}[x=1cm,y=1cm]
\coordinate(center)at(0,0);
\draw[rounded corners=8pt,fill=#1]($(center)+(-0.45,-0.45)$)rectangle($(center)+(0.45,0.45)$);
\ifnum#2=1
\node[circle,draw,fill,inner sep=1pt] at(center){};
\fi
\ifnum#2=2
\node[circle,draw,fill,inner sep=1pt] at($(center)+(-0.2,0)$){};
\node[circle,draw,fill,inner sep=1pt] at($(center)+(0.2,0)$){};
\fi
\ifnum#2=3
\node[circle,draw,fill,inner sep=1pt] at($(center)+(-0.2,-0.2)$){};
\node[circle,draw,fill,inner sep=1pt] at($(center)+(0,0.2)$){};
\node[circle,draw,fill,inner sep=1pt] at($(center)+(0.2,-0.2)$){};
\fi
\ifnum#2=4
\node[circle,draw,fill,inner sep=1pt] at($(center)+(-0.2,0.2)$){};
\node[circle,draw,fill,inner sep=1pt] at($(center)+(0.2,0.2)$){};
\node[circle,draw,fill,inner sep=1pt] at($(center)+(-0.2,-0.2)$){};
\node[circle,draw,fill,inner sep=1pt] at($(center)+(0.2,-0.2)$){};
\fi
\ifnum#2=5
\node[circle,draw,fill,inner sep=1pt] at($(center)+(-0.25,0.25)$){};
\node[circle,draw,fill,inner sep=1pt] at($(center)+(0.25,0.25)$){};
\node[circle,draw,fill,inner sep=1pt] at($(center)+(0,0)$){};
\node[circle,draw,fill,inner sep=1pt] at($(center)+(-0.25,-0.25)$){};
\node[circle,draw,fill,inner sep=1pt] at($(center)+(0.25,-0.25)$){};
\fi
\ifnum#2=6
\node[circle,draw,fill,inner sep=1pt] at($(center)+(-0.2,0.25)$){};
\node[circle,draw,fill,inner sep=1pt] at($(center)+(0.2,0.25)$){};
\node[circle,draw,fill,inner sep=1pt] at($(center)+(-0.2,0)$){};
\node[circle,draw,fill,inner sep=1pt] at($(center)+(0.2,0)$){};
\node[circle,draw,fill,inner sep=1pt] at($(center)+(-0.2,-0.25)$){};
\node[circle,draw,fill,inner sep=1pt] at($(center)+(0.2,-0.25)$){};
\fi
\ifnum#2=7
\node[circle,draw,fill,inner sep=1pt] at($(center)+(-0.15,0.25)$){};
\node[circle,draw,fill,inner sep=1pt] at($(center)+(0.15,0.25)$){};
\node[circle,draw,fill,inner sep=1pt] at($(center)+(-0.25,0)$){};
\node[circle,draw,fill,inner sep=1pt] at($(center)+(0,0)$){};
\node[circle,draw,fill,inner sep=1pt] at($(center)+(0.25,0)$){};
\node[circle,draw,fill,inner sep=1pt] at($(center)+(-0.15,-0.25)$){};
\node[circle,draw,fill,inner sep=1pt] at($(center)+(0.15,-0.25)$){};
\fi
\ifnum#2=8
\node[circle,draw,fill,inner sep=1pt] at($(center)+(-0.25,0.25)$){};
\node[circle,draw,fill,inner sep=1pt] at($(center)+(0.25,0.25)$){};
\node[circle,draw,fill,inner sep=1pt] at($(center)+(-0.25,0)$){};
\node[circle,draw,fill,inner sep=1pt] at($(center)+(0,0.15)$){};
\node[circle,draw,fill,inner sep=1pt] at($(center)+(0,-0.15)$){};
\node[circle,draw,fill,inner sep=1pt] at($(center)+(0.25,0)$){};
\node[circle,draw,fill,inner sep=1pt] at($(center)+(-0.25,-0.25)$){};
\node[circle,draw,fill,inner sep=1pt] at($(center)+(0.25,-0.25)$){};
\fi
\ifnum#2=9
\node[circle,draw,fill,inner sep=1pt] at($(center)+(-0.25,0.25)$){};
\node[circle,draw,fill,inner sep=1pt] at($(center)+(0.25,0.25)$){};
\node[circle,draw,fill,inner sep=1pt] at($(center)+(-0.25,0)$){};
\node[circle,draw,fill,inner sep=1pt] at($(center)+(0,0.25)$){};
\node[circle,draw,fill,inner sep=1pt] at($(center)+(0,0)$){};
\node[circle,draw,fill,inner sep=1pt] at($(center)+(0,-0.25)$){};
\node[circle,draw,fill,inner sep=1pt] at($(center)+(0.25,0)$){};
\node[circle,draw,fill,inner sep=1pt] at($(center)+(-0.25,-0.25)$){};
\node[circle,draw,fill,inner sep=1pt] at($(center)+(0.25,-0.25)$){};
\fi
\end{scope}
\end{tikzpicture}}

\newsavebox{\done}\begin{lrbox}{\done}\dieface{1}\end{lrbox}
\newsavebox{\dtwo}\begin{lrbox}{\dtwo}\dieface{2}\end{lrbox}
\newsavebox{\dthree}\begin{lrbox}{\dthree}\dieface{3}\end{lrbox}
\newsavebox{\dfour}\begin{lrbox}{\dfour}\dieface{4}\end{lrbox}
\newsavebox{\dfive}\begin{lrbox}{\dfive}\dieface{5}\end{lrbox}
\newsavebox{\dsix}\begin{lrbox}{\dsix}\dieface{6}\end{lrbox}
\newsavebox{\dseven}\begin{lrbox}{\dseven}\dieface{7}\end{lrbox}
\newsavebox{\deight}\begin{lrbox}{\deight}\dieface{8}\end{lrbox}
\newsavebox{\dnine}\begin{lrbox}{\dnine}\dieface{9}\end{lrbox}

\newsavebox{\doneR}\begin{lrbox}{\doneR}\dieface[fillred]{1}\end{lrbox}
\newsavebox{\dtwoR}\begin{lrbox}{\dtwoR}\dieface[fillred]{2}\end{lrbox}
\newsavebox{\dthreeR}\begin{lrbox}{\dthreeR}\dieface[fillred]{3}\end{lrbox}
\newsavebox{\dfourR}\begin{lrbox}{\dfourR}\dieface[fillred]{4}\end{lrbox}
\newsavebox{\dfiveR}\begin{lrbox}{\dfiveR}\dieface[fillred]{5}\end{lrbox}
\newsavebox{\dsixR}\begin{lrbox}{\dsixR}\dieface[fillred]{6}\end{lrbox}
\newsavebox{\dsevenR}\begin{lrbox}{\dsevenR}\dieface[fillred]{7}\end{lrbox}
\newsavebox{\deightR}\begin{lrbox}{\deightR}\dieface[fillred]{8}\end{lrbox}
\newsavebox{\dnineR}\begin{lrbox}{\dnineR}\dieface[fillred]{9}\end{lrbox}

\newsavebox{\doneG}\begin{lrbox}{\doneG}\dieface[gold]{1}\end{lrbox}
\newsavebox{\dtwoG}\begin{lrbox}{\dtwoG}\dieface[gold]{2}\end{lrbox}
\newsavebox{\dthreeG}\begin{lrbox}{\dthreeG}\dieface[gold]{3}\end{lrbox}
\newsavebox{\dfourG}\begin{lrbox}{\dfourG}\dieface[gold]{4}\end{lrbox}
\newsavebox{\dfiveG}\begin{lrbox}{\dfiveG}\dieface[gold]{5}\end{lrbox}
\newsavebox{\dsixG}\begin{lrbox}{\dsixG}\dieface[gold]{6}\end{lrbox}
\newsavebox{\dsevenG}\begin{lrbox}{\dsevenG}\dieface[gold]{7}\end{lrbox}
\newsavebox{\deightG}\begin{lrbox}{\deightG}\dieface[gold]{8}\end{lrbox}
\newsavebox{\dnineG}\begin{lrbox}{\dnineG}\dieface[gold]{9}\end{lrbox}

\newcommand{\domsmall}[1]{%
\begin{tikzpicture}[baseline=-5pt,line width=3pt]
\begin{scope}
\coordinate(center)at(0,0);
\ifnum#1=1
\node[circle,draw,fill,inner sep=1pt] at(center){};
\fi
\ifnum#1=2
\node[circle,draw,fill,inner sep=1pt] at($(center)+(-0.2,0)$){};
\node[circle,draw,fill,inner sep=1pt] at($(center)+(0.2,0)$){};
\fi
\ifnum#1=3
\node[circle,draw,fill,inner sep=1pt] at($(center)+(-0.2,-0.2)$){};
\node[circle,draw,fill,inner sep=1pt] at($(center)+(0,0.2)$){};
\node[circle,draw,fill,inner sep=1pt] at($(center)+(0.2,-0.2)$){};
\fi
\ifnum#1=4
\node[circle,draw,fill,inner sep=1pt] at($(center)+(-0.2,0.2)$){};
\node[circle,draw,fill,inner sep=1pt] at($(center)+(0.2,0.2)$){};
\node[circle,draw,fill,inner sep=1pt] at($(center)+(-0.2,-0.2)$){};
\node[circle,draw,fill,inner sep=1pt] at($(center)+(0.2,-0.2)$){};
\fi
\ifnum#1=5
\node[circle,draw,fill,inner sep=1pt] at($(center)+(-0.25,0.25)$){};
\node[circle,draw,fill,inner sep=1pt] at($(center)+(0.25,0.25)$){};
\node[circle,draw,fill,inner sep=1pt] at($(center)+(0,0)$){};
\node[circle,draw,fill,inner sep=1pt] at($(center)+(-0.25,-0.25)$){};
\node[circle,draw,fill,inner sep=1pt] at($(center)+(0.25,-0.25)$){};
\fi
\ifnum#1=6
\node[circle,draw,fill,inner sep=1pt] at($(center)+(-0.2,0.25)$){};
\node[circle,draw,fill,inner sep=1pt] at($(center)+(0.2,0.25)$){};
\node[circle,draw,fill,inner sep=1pt] at($(center)+(-0.2,0)$){};
\node[circle,draw,fill,inner sep=1pt] at($(center)+(0.2,0)$){};
\node[circle,draw,fill,inner sep=1pt] at($(center)+(-0.2,-0.25)$){};
\node[circle,draw,fill,inner sep=1pt] at($(center)+(0.2,-0.25)$){};
\fi
\ifnum#1=7
\node[circle,draw,fill,inner sep=1pt] at($(center)+(-0.15,0.25)$){};
\node[circle,draw,fill,inner sep=1pt] at($(center)+(0.15,0.25)$){};
\node[circle,draw,fill,inner sep=1pt] at($(center)+(-0.25,0)$){};
\node[circle,draw,fill,inner sep=1pt] at($(center)+(0,0)$){};
\node[circle,draw,fill,inner sep=1pt] at($(center)+(0.25,0)$){};
\node[circle,draw,fill,inner sep=1pt] at($(center)+(-0.15,-0.25)$){};
\node[circle,draw,fill,inner sep=1pt] at($(center)+(0.15,-0.25)$){};
\fi
\ifnum#1=8
\node[circle,draw,fill,inner sep=1pt] at($(center)+(-0.25,0.25)$){};
\node[circle,draw,fill,inner sep=1pt] at($(center)+(0.25,0.25)$){};
\node[circle,draw,fill,inner sep=1pt] at($(center)+(-0.25,0)$){};
\node[circle,draw,fill,inner sep=1pt] at($(center)+(0,0.15)$){};
\node[circle,draw,fill,inner sep=1pt] at($(center)+(0,-0.15)$){};
\node[circle,draw,fill,inner sep=1pt] at($(center)+(0.25,0)$){};
\node[circle,draw,fill,inner sep=1pt] at($(center)+(-0.25,-0.25)$){};
\node[circle,draw,fill,inner sep=1pt] at($(center)+(0.25,-0.25)$){};
\fi
\ifnum#1=9
\node[circle,draw,fill,inner sep=1pt] at($(center)+(-0.25,0.25)$){};
\node[circle,draw,fill,inner sep=1pt] at($(center)+(0.25,0.25)$){};
\node[circle,draw,fill,inner sep=1pt] at($(center)+(-0.25,0)$){};
\node[circle,draw,fill,inner sep=1pt] at($(center)+(0,0.25)$){};
\node[circle,draw,fill,inner sep=1pt] at($(center)+(0,0)$){};
\node[circle,draw,fill,inner sep=1pt] at($(center)+(0,-0.25)$){};
\node[circle,draw,fill,inner sep=1pt] at($(center)+(0.25,0)$){};
\node[circle,draw,fill,inner sep=1pt] at($(center)+(-0.25,-0.25)$){};
\node[circle,draw,fill,inner sep=1pt] at($(center)+(0.25,-0.25)$){};
\fi
\draw($(center)+(-0.45,-0.45)$)rectangle($(center)+(0.45,0.45)$);
\end{scope}
\end{tikzpicture}
}

\newsavebox\Mone

\newsavebox{\mycrayon}
\begin{lrbox}{\mycrayon}
\begin{tikzpicture}[line width=3pt,rotate=-20]
\draw[fill=white](0,-2)circle[x radius=0.5cm,y radius=0.2cm];
\draw[fill=white,draw=none](-0.5,0)rectangle(0.5,-2);
\draw[fill=white](0,0)circle[x radius=0.5cm,y radius=0.2cm];
\draw[fill=black](0,0)circle[x radius=0.125cm,y radius=0.04cm];
\draw(-0.5,-2)--(-0.5,0)(0.5,0)--(0.5,-2);
\draw[line width=2pt](-0.5,-2.035)--(0,-2.85)--(0.5,-2.035);
\draw[line width=0pt,fill=black](-0.25,-2.4425)--(0,-2.85)--(0.25,-2.4425);
\draw(0,-0.2)--(0,-2.2)(0.3535,-0.1414)--(0.3535,-2.1414)(-0.3535,-0.1414)--(-0.3535,-2.1414);
\end{tikzpicture}
\end{lrbox}

\newsavebox\mandown
\begin{lrbox}{\mandown}
\begin{tikzpicture}[line width=2pt]
\shade[ball color=fillgreen]
(-0.7,-0.5) 
to [bend right=20](-1.3,-2.5) 
to[bend right=10] (-0.9,-2.48)
to [bend right=5](-0.6,-4.45)
to [bend right=5](0.6,-4.35)
to [bend right=5](1.8,-4.2)
to [out=88,in=-89](1.4,-1.2)
to (1.5,-2)
to [bend right=3](1.9,-1.9)
to [bend right=5](0.9,0.7)
--cycle;
\path
(-0.7,-0.5)edge[bend right=20](-1.3,-2.5)
(-1.3,-2.5)edge[bend right=10](-0.3,-2.4)
(-0.3,-2.4)edge[bend left=10](-0.3,-1.5)
(-0.9,-2.48)edge[bend right=5](-0.6,-4.45)
(-0.6,-4.45)edge[bend right=5](0.6,-4.35)
(0.6,-4.3)edge[bend right=5](1.8,-4.2)
(0.6,-4.3)edge[bend right=0](0.58,-3.2)
(0.4,-3.1)edge[bend left=5](0.76,-3.1)
(1.8,-4.1)edge[out=88,in=-89](1.4,-1.2)
(1.5,-2)edge[bend right=3](1.9,-1.9)
(1.9,-1.9)edge[bend right=5](0.9,0.7)
;
\shade[ball color=fillred,label=E1,draw,line width=1pt](0,0.5)ellipse(1.75cm and 1.75cm);
\draw[fill,rotate=10](-0.9,-0.3)ellipse(2mm and 4mm);
\draw[fill,rotate=-10](0.9,-0.3)ellipse(2mm and 4mm);
\end{tikzpicture}
\end{lrbox}

\newsavebox\nodebox
\begin{lrbox}{\nodebox}
\begin{tikzpicture}
\node[circle,fill=fillred,inner sep=2pt] at (0,0){};
\end{tikzpicture}
\end{lrbox}

\remove{}

\newsavebox{\EMSQRD}
\begin{lrbox}{\EMSQRD}
\begin{tikzpicture}[line width=1.5pt]
\node[draw,rounded corners=10pt,minimum height=1.6cm,minimum width=1.7cm,fill=filllightgray,double,double distance=3pt](B)at(0,0){};
\node[inner sep=1pt](M)at(-0.175,0){\scalebox{2}[3]{\math{\bm{\cl{M}}}}};
\node[inner sep=1pt](sq)at($(M.east|-M.north)+(0.1,-0.1)$)
{\scalebox{1.5}{\bf\textsc{2}}};
\end{tikzpicture}
\end{lrbox}

\newcommand{\TruthTable}[4]{
\begin{tikzpicture}[x=0.65cm,y=0.375cm,baseline=-3pt]
\pgfmathtruncatemacro\R{2^(#1)};
\foreach\x[count=\i]in{#2}{
\node[](v\i)at(0.7*\i,0){\math{\x}};
}
\node[anchor=west](prop)at($(v#1)+(0.5,0)$){#3};
\draw($(v1.west)+(0,-0.5)$)--($(prop.east)+(0,-0.5)$);
\draw($0.5*(v#1.east)+0.5*(prop.west)+(0,0.5)$)--($0.5*(v#1.east)+0.5*(prop.west)+(0,-\R-0.5)$);
\foreach\t[count=\i] in{#4}{
\node[](tv\i)at($(prop.south)+(0,-\i+0.5)$){\t};
}
\foreach\c in {1,...,#1}{
\foreach\r in {1,...,\R}{
\pgfmathtruncatemacro\s{2^(#1-\c)};
\pgfmathtruncatemacro\a{mod(ceil(\r/\s),2)};
\def\tval{\F};
\ifnum\a<1
\def\tval{\T};
\fi
\node[]at(v\c|-tv\r){\tval};
}}
\end{tikzpicture}
}

\newcommand{\TruthTableTwo}[6]{
\begin{tikzpicture}[x=0.65cm,y=0.375cm,baseline=-3pt]
\pgfmathtruncatemacro\R{2^(#1)};
\foreach\x[count=\i]in{#2}{
\node[](v\i)at(0.7*\i,0){\math{\x}};
}
\node[anchor=west](prop1)at($(v#1)+(0.5,0)$){#3};
\node[anchor=west](prop2)at($(prop1.east)+(0.5,0)$){#5};
\draw($(v1.west)+(0,-0.5)$)--($(prop2.east)+(0,-0.5)$);
\draw($0.5*(v#1.east)+0.5*(prop1.west)+(0,0.5)$)--($0.5*(v#1.east)+0.5*(prop1.west)+(0,-\R-0.5)$);
\foreach\t[count=\i] in{#4}{
\node[](tv\i)at($(prop1.south)+(0,-\i+0.5)$){\t};
}
\foreach\t[count=\i] in{#6}{
\node[](tv\i)at($(prop2.south)+(0,-\i+0.5)$){\t};
}
\foreach\c in {1,...,#1}{
\foreach\r in {1,...,\R}{
\pgfmathtruncatemacro\s{2^(#1-\c)};
\pgfmathtruncatemacro\a{mod(ceil(\r/\s),2)};
\def\tval{\F};
\ifnum\a<1
\def\tval{\T};
\fi
\node[]at(v\c|-tv\r){\tval};
}}
\end{tikzpicture}
}
\newcommand{\TruthTableThree}[8]{
\begin{tikzpicture}[x=0.65cm,y=0.375cm,baseline=-3pt]
\pgfmathtruncatemacro\R{2^(#1)};
\foreach\x[count=\i]in{#2}{
\node[](v\i)at(0.7*\i,0){\math{\x}};
}
\node[anchor=west](prop1)at($(v#1)+(0.5,0)$){#3};
\node[anchor=west](prop2)at($(prop1.east)+(0.5,0)$){#5};
\node[anchor=west](prop3)at($(prop2.east)+(0.5,0)$){#7};
\draw($(v1.west)+(0,-0.5)$)--($(prop3.east)+(0,-0.5)$);
\draw($0.5*(v#1.east)+0.5*(prop1.west)+(0,0.5)$)--($0.5*(v#1.east)+0.5*(prop1.west)+(0,-\R-0.5)$);
\foreach\t[count=\i] in{#4}{
\node[](tv\i)at($(prop1.south)+(0,-\i+0.5)$){\t};
}
\foreach\t[count=\i] in{#6}{
\node[](tv\i)at($(prop2.south)+(0,-\i+0.5)$){\t};
}
\foreach\t[count=\i] in{#8}{
\node[](tv\i)at($(prop3.south)+(0,-\i+0.5)$){\t};
}
\foreach\c in {1,...,#1}{
\foreach\r in {1,...,\R}{
\pgfmathtruncatemacro\s{2^(#1-\c)};
\pgfmathtruncatemacro\a{mod(ceil(\r/\s),2)};
\def\tval{\F};
\ifnum\a<1
\def\tval{\T};
\fi
\node[]at(v\c|-tv\r){\tval};
}}
\end{tikzpicture}
}

\newcommand{\venntwo}[2]{
\def\firstcircle{(0,0) circle (1.5cm)}
\def\secondcircle{(0:2cm) circle (1.5cm)}
\def\universal{(-2,-2) rectangle (4,2)}
\foreach\x[count=\i] in {#2}{
\ifnum\i=1
\fill[fill=\x] \universal;
\fi
\ifnum\i=2
\fill[fill=\x] \firstcircle;
\fi
\ifnum\i=3
\fill[fill=\x] \secondcircle;
\fi
\ifnum\i=4
\begin{scope}
\clip\firstcircle;
\fill[fill=\x] \secondcircle;
\end{scope}
\fi
}
\foreach\x[count=\i] in {#1}{
\ifnum\i=1
\draw\firstcircle node[scale=0.9,left=-1pt]{\x};
\fi
\ifnum\i=2
\draw\secondcircle node[scale=0.9,right=-1pt]{\x};
\fi
}
\draw\universal;
}

\newcommand{\vennthree}[2]{
\def\firstcircle{(0,0) circle (1.5cm)}
\def\secondcircle{(60:2cm) circle (1.5cm)}
\def\thirdcircle{(0:2cm) circle (1.5cm)}
\def\universal{(-2,-2) rectangle (4,3.7321)}
\foreach\x[count=\i] in {#2}{
\ifnum\i=1
\fill[fill=\x] \universal;
\fi
\ifnum\i=2
\fill[fill=\x] \firstcircle;
\fi
\ifnum\i=3
\fill[fill=\x] \secondcircle;
\fi
\ifnum\i=4
\fill[fill=\x] \thirdcircle;
\fi
\ifnum\i=5
\begin{scope}
\clip\firstcircle;
\fill[fill=\x] \secondcircle;
\end{scope}
\fi
\ifnum\i=6
\begin{scope}
\clip\firstcircle;
\fill[fill=\x] \thirdcircle;
\end{scope}
\fi
\ifnum\i=7
\begin{scope}
\clip\secondcircle;
\fill[fill=\x] \thirdcircle;
\end{scope}
\fi
\ifnum\i=8
\begin{scope}
\clip\firstcircle;
\clip\secondcircle;
\fill[fill=\x] \thirdcircle;
\end{scope}
\fi
}
\foreach\x[count=\i] in {#1}{
\ifnum\i=1
\draw\firstcircle node[scale=0.9,left=-1pt]{\x};
\fi
\ifnum\i=2
\draw\secondcircle node[scale=0.9,above=-1pt]{\x};
\fi
\ifnum\i=3
\draw\thirdcircle node[scale=0.9,right=-1pt]{\x};
\fi
}
\draw\universal;
}

\newcommand{\mat}[1]{{\ensuremath{\mathrm{#1}}}}
\newcommand{\vect}[1]{{\ensuremath{\mathbf{#1}}}}

\newcommand{\x}{\ensuremath{x}}

\newcommand{\bb}{\vect{b}}
\newcommand{\cc}{\vect{c}}
\newcommand{\ee}{\vect{e}}
\newcommand{\qq}{\vect{q}}
\newcommand{\xx}{\vect{x}}

\newcommand{\uu}{\vect{u}}

\newcommand{\yy}{\vect{y}}

\newcommand{\zz}{\vect{z}}

\newcommand{\vv}{\ensuremath{\mathbf{v}}}

\DeclareMathSymbol{\N}{\mathbin}{AMSb}{"4E}
\DeclareMathSymbol{\Z}{\mathbin}{AMSb}{"5A}
\DeclareMathSymbol{\R}{\mathbin}{AMSb}{"52}
\DeclareMathSymbol{\Q}{\mathbin}{AMSb}{"51}
\DeclareMathSymbol{\I}{\mathbin}{AMSb}{"49}
\DeclareMathSymbol{\C}{\mathbin}{AMSb}{"43}

\newcommand{\poly}{\ensuremath{\textrm{poly}}}
\newcommand{\rank}{\ensuremath{\textrm{rank}}}

\newcommand{\transp}{\ensuremath{^{\text{\textsc{t}}}}}
\newcommand{\trace}{{\textrm{\textup{trace}}}}

\DeclareMathSymbol{\Prob}{\mathord}{AMSb}{"50}

\DeclareMathSymbol{\Exp}{\mathord}{AMSb}{"45}

\def\0{{\bm{0}}}	

\newcommand{\ww}{\vect{w}}

\providecommand\remove[1]{}

\def\matA{\mat{A}}
\def\matB{\mat{B}}
\def\matC{\mat{C}}
\def\matD{\mat{D}}

\def\matI{\mat{I}}
\def\matP{\mat{P}}
\def\matQ{\mat{Q}}
\def\matR{\mat{R}}
\def\matS{\mat{S}}
\def\matT{\mat{T}}
\def\matU{\mat{U}}
\def\matV{\mat{V}}

\def\matX{\mat{X}}

\def\matZ{\mat{Z}}

\newcommand{\smallminus}{{\rule{3pt}{0.3pt}}}
\newcommand{\supminus}{^{\raisebox{2pt}{\hspace*{-0.3pt}\smallminus}}}

\def\matA{\mat{A}}
\def\matI{\mat{I}}
\def\matSig{\mat{\Sigma}}

\def\T{\textsc{t}}
\def\F{\textsc{f}}

\newcommand{\yesbox}{\begin{tikzpicture}[font=\small,baseline=-2.5pt]
\node[draw,inner sep=2pt,rounded corners=3pt] at (-1,0){\textsc{yes}};\end{tikzpicture}}
\def\nobox{\tikz[font=\small,baseline=-2.5pt]{
\node[draw,inner sep=2pt,rounded corners=3pt] at (-1,0){\textsc{no}};}}

\newsavebox\YN
\begin{lrbox}{\YN}
\yesbox{} or \nobox{}
\end{lrbox}
\newsavebox\YBOX
\begin{lrbox}{\YBOX}
\yesbox{}
\end{lrbox}
\newsavebox\NBOX
\begin{lrbox}{\NBOX}
\nobox{}
\end{lrbox}

\newcounter{prooftemplatenum}
\DeclareMathSymbol{\Prob}{\mathbin}{AMSb}{"50}
\DeclareMathSymbol{\Exp}{\mathbin}{AMSb}{"45}

\title{Reduced Label Complexity For Tight \math{\ell_2} Regression}
\author{Alex Gittens\\
 gittea@rpi.edu\\
Computer Science Department\\
Rensselaer Ploytechnic Institute\\
110 8th Street, Troy, NY 12180, USA
  \and Malik Magdon-Ismail\\
 magdon@cs.rpi.edu\\
Computer Science Department\\
Rensselaer Ploytechnic Institute\\
110 8th Street, Troy, NY 12180, USA
}

\begin{document}

\maketitle

\begin{abstract}%
  \noindent
  Given data ${\rm X}\in\bb{R}^{n\times d}$ and labels $\yy\in\bb{R}^{n}$ the goal is find $\ww\in\bb{R}^d$ to minimize $\norm{{\rm X}\ww-\yy}^2$. We give a polynomial algorithm that, \emph{oblivious to $\yy$}, throws out $n/(d+\sqrt{n})$ data points and is a $(1+d/n)$-approximation to optimal in expectation. The motivation is tight approximation with reduced label complexity (number of labels revealed). We reduce label complexity by $\Omega(\sqrt{n})$. Open question: Can label complexity be reduced by $\Omega(n)$ with tight $(1+d/n)$-approximation?
\end{abstract}

\def\r#1{(\ref{#1})}
\section{Introduction}
\label{section:intro}

In an era of big data, upwards of 10 million data points is
not rare. However, labels are costly, especially if humans do
the
labeling. Nevertheless, we want to have and eat our cake. By this
we mean to enjoy the statistical benefits of big data while avoiding the
need for big labeling.

Let \math{\matX\in\R^{n\times d}}
be a data matrix whose rows
are the \math{n} data points, \math{\matX\transp=[\xx_1,\xx_2,\ldots,\xx_n]} and
let \math{\yy\in\R^n} be the corresponding labels,
\math{\yy\transp=[y_1,y_2,\ldots,y_n]}. Typically,
\math{\poly(d)\ll n\ll e^d}.
The age-old goal of \math{\ell_2} regression is to find \math{\ww_*\in\R^d}
satisfying
\mld{
  \norm{\matX\ww_*-\yy}^2\le\norm{\matX\ww-\yy}^2
  \qquad\text{for all \math{\ww\in\R^d}}.
  \label{eq:l2-reg}
}
We study the \emph{label complexity} of solving \r{eq:l2-reg}, the
number of labels in \math{\yy} that must be revealed
to approximate~\math{\ww_*}. To define what
``approximate \math{\ww_*}'' means,
suppose
\math{(\xx_1,y_1),\ldots,(\xx_n,y_n)} are
i.i.d. draws from some joint distribution
\math{D(\xx,y)}. The expected squared prediction error
of \math{\ww_*} approaches optimal, with a statistical error 
\math{O(d/n)}~\cite[Problem 3.11]{malik173}.
Since \math{\ww_*} is only accurate to within \math{O(d/n)}, it suffices
to approximate~\math{\ww_*} to within that same error. It is also necessary
to do so, otherwise the benefit of having big data is lost. This defines the
target approximation regime of interest in large-scale machine learning, one
of the primary consumers of regression.
We seek \math{(1+\epsilon)}-approximations in the regime
\math{\epsilon\le d/n}.
Allowing for randomness, \math{\ww} approximates \math{\ww_*} if
\mld{
  \Exp[\norm{\matX\ww-\yy}^2]\le(1+d/n)\norm{\matX\ww_*-\yy}^2.
  \label{eq:def-approx}
}
Via Markov's inequality, \r{eq:def-approx} implies
a \math{1+O(d/n)} approximation with constant probability.
We give a polynomial approximation algorithm achieving
\r{eq:def-approx} using 
fewer than \math{n} labels, specifically \math{\Omega(\sqrt{n})}
fewer labels. Before stating our result, let us survey the
landscape of tools available, highlighting the need for new tools because
existing methods cannot reduce label complexity in the regime
\math{\epsilon\le d/n}.
There are two settings, consistent regression where
\math{\matX\ww_*=\yy} and inconsistent regression.

{\bf Notation.}
The target matrix \math{\matX} is a fixed
\math{n\times d} real-valued full rank matrix with no zero-rows.
Typically, we will assume \math{\poly(d)\ll n\ll e^d}
when framing asymptotic
runtimes.
Uppercase roman (\math{\matA,\matB,\matC,\matX\ldots}) are 
matrices.
Lowercase bold  (\math{\vect{a},\bb,\cc,\xx,\yy,\zz,\ldots}) are vectors.
We write 
\math{\matX\transp=[\xx_1,\ldots,\xx_n]}, where \math{\xx_i\transp} is the
\math{i}th row of \math{\matX} (the data points).
The standard Euclidean
basis is \math{\ee_1,\ee_2,\ldots} (dimension implied from
context). $\matI_k$ is the $k \times k$ identity and
\math{[k]} is the set \math{\{1,\ldots,k\}}.

The SVD decomposes \math{\matX} into a product, 
\math{\matX=\matU\matSig\matV\transp}. The left-singular matrix
\math{\matU\in\R^{n\times d}} is orthogonal, \math{\matU\transp\matU=\matI_d}.
The \math{i}th leverage score is \math{\ell_i=\norm{\uu_i}^2}, where
\math{\uu_i} is the \math{i}th row of \math{\matU}.
The diagonal matrix \math{\matSig\in\R^{d\times d}_+}
contains the
singular values \math{\sigma_1\ge\sigma_2\ge\cdots\ge\sigma_d>0}.
The right-singular matrix \math{\matV\in\R^{d\times d}} is an orthogonal
rotation.
The SVD can
be computed in time \math{O(nd\min\{n,d\})}.

The Frobenius norm of \math{\matA} is 
\math{
  \norm{\matA}_F^2=\sum_{ij}\matA_{ij}^2=\trace(\matA\transp\matA)=
  \trace(\matA\matA\transp)=\sum_{i\in[d]}\sigma_i^2(\matA).}
The operator or
spectral norm of \math{\matA} is
\math{\norm{\matA}_2=\max_{\norm{\xx}=1}\norm{\matA\xx}=\sigma_1(\matA).}
The condition number of \math{\matX} is
\math{\kappa=\norm{\matA}_2\norm{\matA^{-1}}_2=\sigma_1/\sigma_d}.
The scaled condition number is
\math{\bar\kappa^2=\sum_{i}(\sigma_1/\sigma_d)^2}.

The pseudo-inverse 
\math{\matX^\dagger=(\matX\transp\matX)^{-1}
  \matX\transp=\matV\Sigma^{-1}\matU\transp} provides a solution to
\r{eq:l2-reg}, \math{\ww_*=\matX^\dagger\yy}. The symmetric operator
\math{\matX\matX^\dagger=\matU\matU\transp}
projects onto the column space of \math{\matX}.
For an orthogonal matrix \math{\matQ},
\math{\matQ^\dagger=\matQ\transp} and \math{(\matQ\transp)^\dagger=\matQ}.
We use \math{c,c_1,c_2,\ldots} to generically
denote absolute constants whose values may change with each instance.

\subsection{Consistent (Realizable) \math{\ell_2} Regression}

When \math{\matX\ww_*=\yy}, relative approximation to the optimal in-sample
error is undefined.
In this setting, we require relative
approximation to the optimal weights
\math{\ww_*}.
Pre-conditioning the randomized Kaczmarz algorithm in
\cite{StrohmerVershynin2009}
gives label complexity
\math{d\ln(n\kappa^2/d)} (recall \math{\kappa} is the
conditioning of \math{\matX}).
\begin{theorem}\label{theorem:consistent}
  Set \math{\vv=\bm0}. Independently sample an index \math{j\in[n]} 
  using probabilities
  \math{p_i=\norm{\uu_i}^2/d} for \math{i\in[n]}.
  Do this \math{r} times and for
  each sample perform the
  projective update
  \mld{
    \vv\gets \vv - \frac{\uu_j(\uu_j\transp\vv-y_j)}{\norm{\uu_j}^2}.
  }
  Set \math{\ww=\matV\matSig^{-1}\vv}. Then, for \math{r\ge d\ln(n\kappa^2/d)},
  \mld{
    \Exp[\norm{\ww-\ww_*}^2]\le
    \frac{d}{n}\norm{\ww_*}^2.
  }
\end{theorem}
\noindent
The algorithm in Theorem~\ref{theorem:consistent} requires at most
\math{r} labels, yielding the advertised label complexity.
Direct use of the
result in \cite{StrohmerVershynin2009} gives a label complexity
\math{\bar\kappa^2\ln(n/d)}, where \math{\bar\kappa} is the scaled
condition number, \math{d\le\bar\kappa^2\le1+(d-1)\kappa^2}.
Pre-conditioning brings the
condition number inside the log,
reducing the label complexity from
\math{O(\kappa^2 d\ln(n/d))} to \math{O(d\ln(\kappa^2 n/d))}.
An open question is whether one can remove
dependence on the conditioning all together.

The runtime in Theorem~\ref{theorem:consistent}
is the sum of \math{O(nd^2)} preprocessing to get the leverage
scores \math{\norm{\uu_i}^2} and the pre-conditioner
\math{\matV\matSig^{-1}}, \math{O(r\log n)} to sample the
\math{r} indices and \math{O(rd)} for the
\math{r} projective updates. The \math{O(nd^2)} preprocessing can be
prohibitive. Using the ideas in \cite{malik186},
one can use approximate fast pre-conditioning with
constant factor approximations to the leverage scores to
reduce the preprocessing runtime to \math{O(nd\ln n)}.
The label complexity increases by only a constant factor to
\math{cd\ln(n\kappa^2/d)}, but this
constant factor can be relevant to
practice.

Pre-conditioned SGD with importance sampling
for \math{\ell_p}-regression and minimizing strongly convex functions
has been investigated in some detail
\cite{YangChowReMahoney2016,NeedellWardSrebro2014,GorbunorovHanzelyRichtarik2020}.
Theorem~\ref{theorem:consistent} together with its efficient
extension using
fast approximate
pre-conditioning follows by leveraging ideas from
\cite{StrohmerVershynin2009,YangChowReMahoney2016,malik186}.
This is not a main contribution of
our paper. However,  for completeness,
we give the full analysis  (including identifying the various constants)
in
Appendix~\ref{section:algo}.

\subsection{Inconsistent (Unrealizable) \math{\ell_2} Regression}

Label complexity has received much attention,
especially
in areas such as active learning, \cite{Jacobs2021,MacKay1992}, and experimental design, ~\cite{pukelsheim2006optimal,wang2017computationally,allen2017near}, with
theoretical guarantees being rare on account of the adaptive
sampling of data, \cite{CastroNowak2008}.
There are three general approaches to label complexity.
\begin{enumerate}[label={(\alph*)},itemsep=0pt]
\item
  {\bf Throw away outliers} based on some form of influence
  weights,~\cite{PenaYohai1995}.
  While the practical gains can be considerable, as
  demonstrated in experiments,
all the labels are typically used in determining the outliers and 
theoretical guarantees are lacking.
The typical motivation for identifying outliers
is to improve the expected out-of-sample performance
by ``cleaning'' the data.
This concern is orthogonal to
the main goal of this work whose focus is to
minimize the in-sample
error without using all the in-sample labels.

\item {\bf Iteratively solve (\ref{eq:l2-reg}) using low iteration count.}
  If each iteration touches at most one point, the label complexity is bounded
  by the iteration count. Theorem~\ref{theorem:consistent} uses this approach.
  The state-of-the art in iteration count and efficiency is 
  fast approximate pre-conditioned CGD~\cite{RT08,AvronMaymounkovToledo2010}.
  A data set of size \math{4d^2} is subsampled to construct the preconditioner,
  and then \math{\kappa\ln(n/d)} iterations suffice  to satisfy
  \r{eq:def-approx}. However, the subsampling uses
  random projections to form linear combinations
  of all the data and labels, and each iteration uses all the labels.

  An alternative is to extend the algorithm in
  Theorem~\ref{theorem:consistent} using a
  fast approximate pre-conditioned SGD, as in \cite{YangChowReMahoney2016}.
  While the approach is promising, \math{\Omega(d\log(1/\epsilon)/\epsilon)}
  iterations are needed, resulting in too large a label complexity when 
  \math{\epsilon\le d/n}. 
  
\item
  {\bf Find
  a rich coreset,} a small set of points on which the (possibly reweighted)
  coreset-regression
  approximates the full data regression.
  The active learning paradigm~\cite{MacKay1992,CohnAtlasLadner1994,FreundSeungShamirTishby1997}
  adds
  one point at a time adaptively to the working coreset. 
  This adaptive sampling can exponentially reduce
  label complexity in classification
  from \math{d/\epsilon} to \math{d\log(d/\epsilon)}.
  However, the settings are very restricted,
  such as consistent (separable)
  homogeneous linear models with data uniform on the sphere,
  \cite{FreundSeungShamirTishby1997,SanguptaKalaiMonteleoni2005,BalcanBeygelzimerLangford2006,BalcanBroderZhang2007}. Even mild deviation from these settings
  can result in label complexity reverting to \math{d/\epsilon}
  \cite{Dasgupta2005}. As with outlier ejection,
  active learning in machine learning is focused on out-of-sample
  prediction error for a test distribution. Our focus is tight
  in-sample fit with minimum label complexity. To this end,
  one fast random projections  
  efficiently construct coresets of size \math{O(d/\epsilon)}, \cite{S2006},
  but these coresets are
  linear combinations of all the data. The motivation of random-projection
  coresets is speed, not label
  complexity.  
  Row-sampling according to
  leverage score probabilities \cite[Theorem 5]{DMM2008} uses a pure coreset
  of size \math{\Omega(d\log d/\epsilon^2)} to produce
  a \math{(1+\epsilon)}-approximator with constant probability.
  In a sequence of ensuing
  results
  using more refined approaches, this sample complexity has
  been reduced. First one can start with a constant factor approximation
  using \math{O(d\log d)} samples and improve that to a
  \math{(1+\epsilon)}-approximation using an additional
  \math{d/\epsilon} samples. This improves the result in
  \cite{DMM2008} to \math{O(d\log d+d/\epsilon)}~\cite{Mahoney2011}.
  The \math{d\log d} is unavoidable by a coupon collector argument. However,
  using a more subtle linear sample sparsification approach,
  \cite{ChenPrice2019} gets the row-sample complexity down to
  \math{O(d)} for a 2-approximation, which then gives an
  \math{O(d+d/\epsilon)} label complexity for
  a \math{(1+\epsilon)}-approximation in expectation, the current state of
  the art. An interesting result in \cite{DerezenskiWarmuth2018}
  uses volume sampling
  to obtain an unbiased \math{(d+1)}-approximation using
  \math{d} labels, assuming the data in \math{\matX} are in general
  position (there is no easy way to extend this analysis to sample more than
  \math{d} points). Volume sampling has also been used in matrix reconstruction
  \cite{DRVW2006}.
  The estimator in \cite{DerezenskiWarmuth2018} is unbiased, hence
  averaging
  gives a
  \math{(1+\epsilon)}-approximator with \math{O(d^2/\epsilon)} labels.
  \cite{DerezenskiWarmuth2018} emphasize that
  jointly sampling rows is essential for
  getting tight approximation, and then go on
  to give an efficient algorithm
  for reverse iterative volume sampling, improving on the volume sampling
  algorithms in
  \cite{DR2010,KuleszaTasker2012}. Note that only coresets constructed oblivious
  to the labels \math{\yy} can reduce label complexity,
  for example \cite{ChenPrice2019,DerezenskiWarmuth2018}.
\end{enumerate}
The prior results don't work in the
stringent \math{\epsilon\le d/n} regime,
since they imply label complexity
\math{n}. New tools are needed for this
regime.
We give a polynomial
algorithm to reduce label complexity by \math{\Omega(\sqrt{n})}.
Our algorithm throws away data while maintaining a provable
coreset (a combination of
approaches (a) and (c) above). The algorithm is based on the following new tools:
\begin{enumerate}[itemsep=0pt,label={(\roman*)}]
\item Tight analysis of the regression error obtained by solving the
  regression problem on an \emph{arbitrary} coreset obtained
  after throwing away
  \math{k} points.
\item A probabilistic argument showing that one can always
  throw away \math{\Omega(\sqrt{n})} points while achieving
  the target approximation error, reducing label complexity by
  \math{\Omega(\sqrt{n})}.
\item The probabilistic arguments use a counterintuitive
  sampling measure for sets of rows.
  To realize the reduced label complexity implied by (\rn{2}),
  we develop a polynomial rejection sampling algorithm to throw out a 
  set of size \math{\Omega(\sqrt{n})} while attaining the bound
  in
  \r{eq:def-approx}.
\end{enumerate}

\noindent {\bf Lower Bounds.}
Theorems~13 and 14 in \cite{malik190} give lower bounds on
\math{\yy}-agnostic coresets with
\math{1+d/n} approximation ratio. Deterministic 
\math{\yy}-agnostic coresets have at least \math{n-d} points
(label complexity cannot be reduced more than \math{d}).
Randomized \math{\yy}-agnostic coresets yielding
\math{1+d/n} approximation with constant probability have
at least \math{n/d} points, so label complexity cannot be reduced by more than
\math{cn}, where \math{c\sim (d-1)/d}.
Since \r{eq:def-approx} implies approximation with constant probability,
the maximum
reduction in label complexity one can hope for is \math{cn}.

\subsection{Our results}

Let \math{\matA} be a matrix formed from a \math{k}-subset of the rows in
\math{\matX}. That is,
\math{
  \matA=\matS\transp\matX,
}
where \math{\matS} is a row-sampling matrix whose columns are standard
basis vectors,
\math{
  \matS=
      [\ee_{i_1},\ee_{i_2},\ldots,\ee_{i_k}].
}
Recall that \math{\matX=\matU\matSig\matV\transp} and
let \math{\matU_\matA=\matS\transp\matU} be the
corresponding rows of \math{\matU}.
The partial projection matrix
\math{\matP_\matA} plays an important role in our algorithm,
\mld{
  \matP_\matA
  =
  \matA(\matX\transp\matX)^{-1}\matA\transp
  =
  \matU_\matA\matU_\matA\transp,
}
where the last expression follows from using the SVD of \math{\matX}, see
\r{eq:partial-proj}.
The 
\emph{influence} of the rows \math{\matA} is
\mld{
  p_\matA=\frac{1}{\cl{Z}}\frac{(1-\norm{\matP_\matA}_2)^2}{\norm{\matP_\matA}_2},
    \label{eq:intro:pA}
}
where \math{\cl{Z}=\sum_\matA(1-\norm{\matP_\matA}_2)^2/\norm{\matP_\matA}_2}.
Note that the influence does not depend on the labels \math{\yy}. Our main
result is Theorem~\ref{thm:k-points}, and the algorithm
accompanying  Theorem~\ref{thm:k-points} is simple to state.
Jointly sample \math{k} rows \math{\matA} to throw out, using the
probability distribution over \math{k}-subsets of the
rows in \math{\matX} given by the influences
\math{p_\matA} in \r{eq:intro:pA}. Let \math{\matX\supminus_\matA} be
the (deficient) data that remains after throwing out the \math{k} rows in
\math{\matA}, and let \math{\yy\supminus_\matA} be the corresponding
labels. Perform a simple regression on this reduced (deficient) data to
get regression weights \math{\ww\supminus_\matA}. Then,
Theorem~\ref{thm:k-points} states that
\mld{
    \Exp[\norm{\matX\ww_\matA\supminus-\yy}^2]
    \le
    \left(1+\frac{dk^2}{(n-dk)^2}\right)
    \norm{\matX\ww_*-\yy}^2.
    \label{eq-intro:main-theorem}
}
Note that the algorithm is oblivious to \math{\yy} and hence serves to
reduce the label complexity by \math{k} while delivering the approximation in
\r{eq-intro:main-theorem}.
The main tool in our analysis is Lemma~\ref{lem:leave-A-out-error}
which gives an exact
analysis of the regression obtained from throwing away
an arbitrary set of rows \math{\matA}.

When \math{k=1}, the algorithm
throws out one data point \math{\xx_i} using sampling probabilities
(influences) \math{p_i\propto(1-\ell_i)^2/\ell_i}.
By throwing out \math{(1/n)}th of the
information, one expects the error to grow correspondingly, by \math{1/n}.
A surprise from
\r{eq-intro:main-theorem} is that one can throw away
one data point and get only an \math{O(d/n^2)} error increase.
Prior algorithms that explicitly construct coresets 
can't guarantee such approximations for coreset
sizes smaller than \math{n}. This already breaks a barrier on what was
previously possible.
Setting
\math{{dk^2}/{(n-dk)^2}=d/n} proves that one can throw out
\math{k=n/(d+\sqrt{n})\in\Omega(\sqrt{n})} data points and get approximation
ratio \math{1+d/n}. In Section~\ref{sec:reduction-one} we prove the
result for \math{k=1} illustrating all the main ideas, which are then
generalized in
Section~\ref{sec:reducing-root-n}.

As it stands, \r{eq-intro:main-theorem} is an existence result,
unless one can efficiently sample \math{\matA} according
to \math{p_\matA}. The probabilities
\math{p_\matA} depend non-trivially
on \math{\matA} through the spectral norm of \math{\matU_\matA}, and there
is no obvious way to jointly sample rows using
such complicated probabilities.
In Section~\ref{sec:sampling} we give an algorithm to sample exactly from
the probabilities \math{p_\matA}. The runtime to generate one sample
\math{\matA} satisfying \r{eq-intro:main-theorem} is
\math{O(\mu(n+kd\min\{k,d\}))}, where \math{\mu} is the average
inverse leverage score, a measure of coherence,
\mld{
  \mu(\matX)=\frac{1}{n}\sum_{i=1}^n\frac{1}{\ell_i}.
}
(\math{\ell_i} are the leverage scores, \math{\ell_i=\norm{\uu_i}^2}.)
For near uniform leverage scores, \math{\mu\sim n/d}, and
the
runtime is \math{O(n^2/d)}.
Ideally, the sampling efficiency should not depend on the input.

The sampling algorithm uses two tools. The first is
Theorem~\ref{theorem:sampling} which is a simple way to sample using
probabilities \math{p_\matA} that
can be written as a sum of some function over the
rows of \math{\matA}, for example sampling according to
Frobenius norms, \math{p_\matA\propto\norm{\matA}^2_F}.
Our sampling probabilities cannot be written
as a sum over rows, which leads to
our second idea of carefully bounding the
sampling probabilities so that we can use Theorem~\ref{theorem:sampling}
within a rejection sampling framework. 

{\bf Remainder of the paper.}
Next, we briefly discuss
some open questions and promising directions. We then proceed to
the detailed statement of results and proofs.

\subsection{Discussion}

Our result for the special case \math{k=1} 
highlights the need for new tools when the approximation regime
is stringent. Constructing a coreset from scratch
via some form
of sparsification  won't work.
Carefully throwing away data does work.
It is instructive to see what our result in \r{eq-intro:main-theorem}
implies for
a \math{(1+\epsilon)}-approximation in the more relaxed setting
where \math{\epsilon} is a
(small) constant. Setting \math{dk^2/(n-dk)^2=\epsilon}
gives \math{k=n/(d+\sqrt{d/\epsilon})}, so our algorithm throws away
\math{O(n/d)} data, retaining a coreset proportional to
\math{n}. This is much worse than the coreset construction
algorithms based on sparsification which only need to retain
\math{O(d/\epsilon)} points. Coreset construction is better for relaxed
approximation and data rejection is better for tight approximation.
It is not unusual for different
regimes to require different techniques. However it is an open question
whether data rejection can compete with
coreset construction even for relaxed approximation.

Our algorithm throws out \math{\Omega(\sqrt{n})}
data and provably gets a \math{(1+d/n)}-approximation.
There are reasons to suspect that one can throw out
\math{cn} data points and get \math{(1+d/n)}-approximation.
\begin{inparaenum}[(i)]
  \item The lower bound suggests one only needs to retain
    \math{n/d} data points, hence throwing out \math{(d-1)n/d}.
  \item
    If one can repeatedly
    throw out one point with the \math{k=1} result continuing to hold in
    a chaining fashion, one can throw out proportional
    to \math{n} data points (see the comments
    after Theorem~\ref{thm:one-point}).
    Unfortunately,
    the chaining analysis, being adaptive, is difficult.
\end{inparaenum}

Lemma~\ref{lem:leave-A-out-error} is an exact 
leave-\math{\matA}-out result. Our analysis then bounds
\math{(\matA\ww_*-\yy_\matA)\transp
    \matQ
    (\matA\ww_*-\yy_\matA)} by
\math{\norm{\matQ}_2\norm{\matA\ww_*-\yy_\matA}^2}. This is loose because
it does not exploit the coordination between the
residual \math{\matA\ww_*-\yy_\matA} and \math{\matQ}. In
special cases, e.g. \math{d=1}, one can
exploit this coordination
to throw out \math{cn} data points
and get \math{(1+d/n)}-approximation, matching the
lower bound. Hence, a more subtle analysis
could resolve our main open question of whether one can throw out
\math{cn} points and get \math{(1+d/n)}-approximation.
We used a simple regression for inference on the deficient data.
A different inference algorithm might produce stronger results,
for example a weighted
regression as is used in the coreset construction.
Or, an all together new approach is needed.

The (oblivious to \math{\yy})
influence probabilities in \r{eq:intro:pA}
identify the ``useless'' rows, akin to outlier detection.
The innovation in our algorithm is that the useless rows are
\emph{jointly} sampled. For coreset construction, joint
sampling of rows is essential to get the tightest bounds, and the
same is likely true for identifying the useless rows.
Thus, the probablities \math{p_\matA} in \r{eq:intro:pA}
may be of general interest to machine learning.
Can one more efficiently
sample according to complex probabilities like \math{p_\matA}? Or,
are there approximations to \math{p_\matA} that can give the same
regression accuracy but are easier to sample from?
How does the bound change if
approximate sampling probabilities are used instead of \math{p_\matA}?

This work addresses the transductive setting,
where one simply wishes to obtain the optimal in-sample weights
\math{\ww_*}, but using fewer labels. In the inductive setting,
one is also interested in the expected prediction error on new data
\math{(\xx,y)} drawn from some distribution. It would be interesting
to understand how rejection performs in the inductive setting.

\section{Reducing Label Complexity By One}
\label{sec:reduction-one}

Recall that  a \math{k}-subset of the rows in
\math{\matX} is 
\math{\matA=\matS\transp\matX,}
where
\math{\matS=[\ee_{i_1},\ee_{i_2},\ldots,\ee_{i_k}].}
Let \math{\yy_\matA} be the corresponding \math{y}-values for the data in
\math{\matA}, \math{\yy_\matA=\matS\transp\yy}.
Using the notation in \cite[Section 4.3]{malik173},
define \math{\matX_{\matA}\supminus} as the
deficient dataset with the rows in
\math{\matA} removed. Similarly, we
have \math{\yy_\matA\supminus}, the corresponding \math{y}-values for the
deficient data and \math{\ww_{\matA}\supminus}, the regression weights
obtained from the deficient data,
\mld{
  \ww_{\matA}\supminus
  =
  \argmin_\ww\norm{\matX_{\matA}\supminus\ww-\yy_\matA\supminus}^2.
}
The partial projection matrix
\math{\matP_\matA} has an important role in our discussion,
\mld{
  \matP_\matA
  =
  \matA(\matX\transp\matX)^{-1}\matA\transp.
  \label{eq:partial-proj0}
}
Recall the SVD of \math{\matX}, \math{\matX=\matU\matSig\matV\transp}.
Let \math{\matU_\matA} be the rows in \math{\matU} corresponding to the
rows \math{\matA}. Then,
\mld{
  \matP_\matA=\matS\transp\matX(\matX\transp\matX)^{-1}\matX\transp\matS
  =
  \matS\transp\matU\matU\transp\matS
  =\matU_\matA\matU_\matA\transp,
\label{eq:partial-proj}}
and
hence \math{\norm{\matP_\matA}_2=\norm{\matU_\matA\matU_\matA\transp}_2\le1}.
Assume \math{\norm{\matP_\matA}_2<1}. This will be
without loss of generality because we never need to remove a set of
rows \math{\matA} with \math{\norm{\matP_\matA}_2=1}.
Also assume \math{0<\norm{\matP_\matA}_2} because if
\math{0=\norm{\matP_\matA}_2} for any \math{\matA}, those
rows in \math{\matA}
are all \math{\bm{0}} and can be thrown out.
We need the in-sample error
for the deficient weights \math{\ww_A\supminus} on the full data
\math{\matX}. This is the content of the next lemma,
\begin{lemma}\label{lem:leave-A-out-error}
  Let \math{\ww_*} be the weights from the full regression,
  \math{\ww_*=\argmin_\ww\norm{\matX\ww-\yy}^2}.
  \mld{
    \norm{\matX\ww_\matA\supminus-\yy}^2
    =
    \norm{\matX\ww_*-\yy}^2
    +
    (\matA\ww_*-\yy_\matA)\transp
    \matQ
    (\matA\ww_*-\yy_\matA),
  }
  where,
  assuming \math{\norm{\matP_\matA}_2<1},
  \math{\matQ=(\matI_k-\matP_\matA)^{-1}\matP_\matA(\matI_k-\matP_\matA)^{-1}=(\matI_k-\matP_\matA)^{-2}-
      (\matI_k-\matP_\matA)^{-1}.}
\end{lemma}
\begin{proof}
  Note that
  \math{{\matX_{\matA}\supminus}\transp\matX_{\matA}\supminus
    =
    \matX\transp\matX-\matA\transp\matA},
  and
  \math{{\matX_{\matA}\supminus}\transp\yy_\matA\supminus
    =
    \matX\transp\yy-\matA\transp\yy_\matA}. The deficient weights
  \math{\ww_\matA\supminus} are
  \mld{
    \ww_\matA\supminus
    =
    ({\matX_{\matA}\supminus}\transp\matX_{\matA}\supminus)^{-1}
    {\matX_{\matA}\supminus}\transp\yy_\matA\supminus
    =
   ({\matX_{\matA}\supminus}\transp\matX_{\matA}\supminus)^{-1}
    (\matX\transp\yy-\matA\transp\yy_\matA)
  .
  }  
  Using \math{\ww_*=(\matX\transp\matX)^{-1}\matX\transp\yy} and
  the Woodbury matrix inversion identity~\cite{Woodbury1950},
  \eqar{
    \ww_\matA\supminus
    &=&
    (\matX\transp\matX-\matA\transp\matA)^{-1}
    (\matX\transp\yy-\matA\transp\yy_\matA)
    \\
    &=&
    \left[(\matX\transp\matX)^{-1}+
      (\matX\transp\matX)^{-1}\matA\transp(\matI_k-\matP_\matA)^{-1}
      \matA(\matX\transp\matX)^{-1}
      \right]
    (\matX\transp\yy-\matA\transp\yy_\matA)
    \\
    &=&
    \ww_*+(\matX\transp\matX)^{-1}\matA\transp(\matI_k-\matP_\matA)^{-1}
    \matA\ww_*
    -    (\matX\transp\matX)^{-1}\matA\transp(\matI_k-\matP_\matA)^{-1}\yy_\matA
    ,
  }
  where the last expression follows by multiplying out the previous
  expression and using
  \mld{
    (\matI_k-\matP_\matA)^{-1}\matP_\matA
    =
    (\matI_k-\matP_\matA)^{-1}(\matP_\matA-\matI+\matI)=(\matI_k-\matP_\matA)^{-1}-\matI.
    \label{eq:proj-identity}}
  Note, \math{\ww_\matA\supminus} is well defined since
  \math{\matI_k-\matP_\matA} is invertible because
  \math{\norm{\matP_\matA}_2<1}. Consider
  \math{\norm{\matX\ww_\matA\supminus-\yy}^2},
  \mld{
    \norm{\matX\ww_\matA\supminus-\yy}^2
    =
    \norm{\matX\ww_*-\yy+\matX(\matX\transp\matX)^{-1}\matA\transp(\matI_k-\matP_\matA)^{-1}
    \matA\ww_*
    -\matX
    (\matX\transp\matX)^{-1}\matA\transp(\matI_k-\matP_\matA)^{-1}\yy_\matA}^2.
  }
  We get the norms-squared of each of the three terms, plus the cross terms.
  The residual \math{\matX\ww_*-\yy} is orthogonal to the columns
  of \math{\matX}, that is \math{(\matX\ww_*-\yy)\transp\matX=\bm{0}}.
  Hence, only one of
  the cross terms is non-zero. After a little algebra, we get four terms,
  \eqar{
    \norm{\matX\ww_\matA\supminus-\yy}^2
    &=&
    \norm{\matX\ww_*-\yy}^2
    \nonumber
    \\
    &&
    +\ww_*\transp\matA\transp(\matI_k-\matP_\matA)^{-1}\matP_\matA(\matI_k-\matP_\matA)^{-1}\matA\ww_*
    \nonumber
    \\
    &&
    +\yy_\matA\transp(\matI_k-\matP_\matA)^{-1}\matP_\matA(\matI_k-\matP_\matA)^{-1}
    \yy_\matA
    \nonumber
    \\
    &&
    -2\ww_*\transp\matA\transp(\matI_k-\matP_\matA)^{-1}\matP_\matA(\matI_k-\matP_\matA)^{-1}\yy_\matA.
    \\
    &=&
    \norm{\matX\ww_*-\yy}^2
    +
    (\matA\ww_*-\yy_\matA)\transp
    (\matI_k-\matP_\matA)^{-1}\matP_\matA(\matI_k-\matP_\matA)^{-1}
    (\matA\ww_*-\yy_\matA).
    }
The alternate form for \math{\matQ} follows by using \r{eq:proj-identity}.
\end{proof}
\noindent
A special case of Lemma~\ref{lem:leave-A-out-error} is when \math{k=1}
(one row is
removed). The general case uses similar ideas. When \math{\matA} is just one
row, 
\math{\matA=\xx_i\transp}, and
\math{\matP_\matA=\uu_i\transp\uu_i=\norm{\uu_i}^2=\ell_i,}
the leverage score for the \math{i}th
row of \math{\matX} (norm-squared of the corresponding
row of the left-singular matrix).  Lemma~\ref{lem:leave-A-out-error} gives
\mld{
  \norm{\matX\ww_i\supminus-\yy}^2
  =
  \norm{\matX\ww_*-\yy}^2
  +
  \frac{\ell_i}{(1-\ell_i)^2}\norm{\xx_i\transp\ww_*-\yy_i}^2
}
Define the sampling probability
\mld{
  p_i=\frac{1}{\cl{Z}}\frac{(1-\ell_i)^2}{\ell_i},
}
where \math{\cl{Z}=\sum_i(1-\ell_i)^2/\ell_i}. Sample a row \math{i}
with probability
\math{p_i} to throw out. Notice that if \math{\ell_i=1} then this row will
never be thrown out,
consistent with our assumption that \math{\norm{\matP_\matA}_2<1}.
\begin{theorem}\label{thm:one-point}
  For any \math{\matX},
  pick row \math{i} to throw out with probability \math{p_i}. Then,
  \mld{
    \Exp[\norm{\matX\ww_i\supminus-\yy}^2]
    \le
    \left(1+\frac{d}{(n-d)^2}\right)
    \norm{\matX\ww_*-\yy}^2.
    }
\end{theorem}
\noindent
Theorem~\ref{thm:one-point} implies one can
throw out at least one point and get a
\math{(1+O(d/n^2))}-approximation. Recall that our
target approximation \math{1+d/n}.
\begin{proof}
  Using the definition of  \math{p_i} and
  \math{\sum_i\norm{\xx_i\transp\ww_*-\yy}^2=\norm{\matX\ww_*-\yy}^2}
  gives
  \mld{
    \Exp[\norm{\matX\ww_i\supminus-\yy}^2]
    =
    \left(1+\frac{1}{\cl{Z}}\right)
    \norm{\matX\ww_*-\yy}^2.
  }
  The result follows if \math{\cl{Z}\ge(n-d)^2/d}, which we now
  prove.
  \eqar{
    \cl{Z}
    &=&
    \sum_{i=1}^n\frac{(1-\ell_i)^2}{\ell_i}
    =
    \sum_{i=1}^n\frac{1}{\ell_i}-2+\ell_i.
  }
  Since \math{\sum_i\ell_i=\sum_i\norm{\uu_i}^2=d},
  \math{\cl{Z}=d-2n+\sum_i1/\ell_i}. Therefore, we wish to find the minimum
  possible value of \math{\sum_i1/\ell_i} subject to the constraint
  \math{0<\ell_i\le 1} and \math{\sum_i\ell_i=d}. Let us suppose that this
  minimum is attained at some values \math{\ell_{1*},\ldots,\ell_{n*}} and
  for some \math{i,j}, \math{\ell_{i*}<\ell_{j*}}. Suppose
  \math{\ell_{i*}=\ell-\varepsilon} and \math{\ell_{j*}=\ell+\varepsilon}.
  After some elementary algebra, one
  finds that replacing both \math{\ell_{i*}} and \math{\ell_{j*}} by
  \math{\ell} keeps their sum the same but strictly decreases the sum
  \math{1/\ell_{i*}+1/\ell_{j*}},
  \mld{
    \frac{1}{\ell-\varepsilon}+\frac{1}{\ell+\varepsilon}
    >
    \frac{2}{\ell}.
  }
  This contradicts
  \math{\ell_{1*},\ldots,\ell_{n*}} attaining the minimum for \math{\cl{Z}}
  which means the minimum possible value for \math{\cl{Z}} is attained when
  \math{\ell_{1*}=\ell_{2*}=\cdots=\ell_{n*}=d/n}. This gives
  \mld{\cl{Z}\ge n\times\frac{(1-d/n)^2}{(d/n)}=\frac{(n-d)^2}{d},}
  concluding the proof.
\end{proof}
\noindent
{\bf Comment.} Throwing out just one data point looks trivial,
but the result is surprising. A data point contains \math{O(1/n)}
of the information yet throwing one out increases the error by only
\math{O(1/n^2)}.

\noindent
{\bf Comment.} The same qualitative relative error approximation
\math{1+cd/(n-d)^2} continues to hold given relative error
approximations to \math{\ell_i} and \math{(1-\ell_i)}. A fast relative error
approximation to \math{\ell_i} is given in~\cite{malik186}.
Can one can get a fast
relative error approximation to \math{(1-\ell_i)}?

\noindent
{\bf Comment.} The algorithm is oblivious to \math{\yy} as it
should be if we are to reduce label complexity.

\noindent
{\bf Comment.} Chaining this approximation factor by 
throwing out one point at a time gives
\mld{
  \left(1+\frac{d}{(n-d)^2}\right)
  \left(1+\frac{d}{(n-d-1)^2}\right)
  \left(1+\frac{d}{(n-d-2)^2}\right)
  \cdots
  \left(1+\frac{d}{(n-d-k+1)^2}\right).
}
Using \math{1+x\le e^x}, this product is at most
\mld{
  \exp\left(d\sum_{i=0}^{k-1}\frac{1}{(n-d-i)^2}\right).
}
Bounding the sum by an integral gives
\mld{
  \sum_{i=0}^{k-1}\frac{1}{(n-d-i)^2}\le\int_0^kdx\ \frac{1}{(n-d-x)^2}
  =\frac{k}{(n-d)(n-d-k)}.}
Setting \math{k=(n-d)/2} 
gives an approximation ratio
\math{\exp\left({d}/{(n-d)}\right)\approx 1+{d}/{(n-d)}}
after throwing away about half the data.
Such a result ought to be
possible, but we don't have a proof for any such
chaining approach.
Our general analysis in the next section only throws out
\math{\Theta(\sqrt{n})} points.

\section{Reducing Label Complexity by \math{\Omega(\sqrt{n})}.}
\label{sec:reducing-root-n}

The goal in this section is to show that one can reduce
label complexity by
\math{\Omega(\sqrt{n})} while attaining the target approximation ratio of
\math{1+d/n}. We prove that such a set of rows \math{\matA}
exists and give a polynomial algorithm to find \math{\matA}.
The starting point
is Lemma~\ref{lem:leave-A-out-error}, which implies
\mld{
  \norm{\matX\ww_\matA\supminus-\yy}^2
    \le
    \norm{\matX\ww_*-\yy}^2
    +
    \norm{(\matI_k-\matP_\matA)^{-1}\matP_\matA
  (\matI_k-\matP_\matA)^{-1}}_2
    \norm{\matA\ww_*-\yy_\matA}^2
    .
}
Let \math{0\le\lambda<1} be an eigenvalue of
\math{\matP_\matA}. Then, 
\mld{
    {\lambda}/{(1-\lambda)^2}\ge0
}
is an eigenvalue of \math{(\matI_k-\matP_\matA)^{-1}\matP_\matA
  (\matI_k-\matP_\matA)^{-1}}.
We see that \math{(\matI_k-\matP_\matA)^{-1}\matP_\matA
  (\matI_k-\matP_\matA)^{-1}} is positive, hence
\math{\norm{
      (\matI_k-\matP_\matA)^{-1}\matP_\matA
  (\matI_k-\matP_\matA)^{-1}}_2}
is given by its top eigenvalue, which is obtained from the
top eigenvalue of \math{\matP_\matA}, which in turn is
\math{\norm{\matP_\matA}_2} since \math{\matP_\matA} is non-negative. Hence,
\mld{
  \norm{\matX\ww_\matA\supminus-\yy}^2
    \le
    \norm{\matX\ww_*-\yy}^2
    +
    \frac{\norm{\matP_\matA}_2}{(1-\norm{\matP_\matA}_2)^2}
    \norm{\matA\ww_*-\yy_\matA}^2
    .
    \label{eq:leave-A-out}
}
Define a sampling probability for a subset of rows \math{\matA} by
\mld{
  p_\matA=\frac{1}{\cl{Z}}\frac{(1-\norm{\matP_\matA}_2)^2}{\norm{\matP_\matA}_2},
    \label{eq:pA}
}
where \math{\cl{Z}=\sum_\matA(1-\norm{\matP_\matA}_2)^2/\norm{\matP_\matA}_2}. Sample the set of \math{k} rows \math{\matA} to throw out
with probability
\math{p_\matA}.
Note that we never throw out an \math{\matA} with
\math{\norm{\matP_\matA}_2=1},
consistent with assuming \math{\norm{\matP_\matA}_2<1}.
\begin{theorem}\label{thm:k-points}
  For any \math{\matX},
  pick \math{k} rows \math{\matA}
  to throw out with probability \math{p_\matA}. Then,  for \math{k<n/d},
  \mld{
    \Exp[\norm{\matX\ww_\matA\supminus-\yy}^2]
    \le
    \left(1+\frac{dk^2}{(n-dk)^2}\right)
    \norm{\matX\ww_*-\yy}^2.
    }
\end{theorem}
Set \math{k=n/(d+\sqrt{n})} in Theorem~\ref{thm:k-points} to get
a \math{(1+d/n)}-approximation. That is, one can reduce
label complexity by \math{n/(d+\sqrt{n})\in\Omega(\sqrt{n})} while attaining
the target approximation ratio.
\begin{proof}
  Taking the expectation in \r{eq:leave-A-out} using the probabilities in
  \r{eq:pA} gives
  \mld{
      \Exp[\norm{\matX\ww_\matA\supminus-\yy}^2]
    \le
    \norm{\matX\ww_*-\yy}^2
    +
    \frac{1}{\cl{Z}}\sum_{\matA}
    \norm{\matA\ww_*-\yy_\matA}^2
    .
    \label{eq:proof-leave-A-out-1}
  }
  Let us evaluate the sum over \math{\matA} in \r{eq:proof-leave-A-out-1}.
  Fix \math{i\in[n]}. The term \math{(\xx_i\transp\ww_*-y_i)^2} appears in
  \math{\choose{n-1}{k-1}} of the \math{\matA}s. Hence the sum over
  \math{\matA} contains \math{\choose{n-1}{k-1}} copies of
  \math{(\xx_i\transp\ww_*-y_i)^2} for each \math{i}. This means
  \mld{    
  \sum_{\matA}
    \norm{\matA\ww_*-\yy_\matA}^2
    =
    \choose{n-1}{k-1}\norm{\matX\ww_*-\yy}^2
  }
  and we get
  \mld{
      \Exp[\norm{\matX\ww_\matA\supminus-\yy}^2]
    \le
    \left(1+\frac{\choose{n-1}{k-1}}{\cl{Z}}\right)\norm{\matX\ww_*-\yy}^2
    .
    \label{eq:proof-leave-A-out-2}
  }
  The remainder of the proof is to upperbound
  \math{{\choose{n-1}{k-1}}/{\cl{Z}}}. We need a lower bound on \math{\cl{Z}}.
  \eqar{
    \cl{Z}
    &=&
    \sum_{\matA}\frac{(1-\norm{\matP_\matA}_2)^2}{\norm{\matP_\matA}_2}
    =
    \sum_{\matA}\frac{1}{\norm{\matP_\matA}_2}
    +
    \sum_{\matA}{\norm{\matP_\matA}_2}
    -2n.
  }
  As in the proof of Theorem~\ref{thm:one-point}, fix the sum
  \math{\sum_{\matA}{\norm{\matP_\matA}_2}}. Then the sum
  \math{\sum_{\matA}{1}/{\norm{\matP_\matA}_2}} is minimized when each
  term has the same value, i.e., 
    \math{\norm{\matP_\matA}_2=\sum_{\matA}{\norm{\matP_\matA}_2}/\choose{n}{k}},
    the average
  spectral norm of the partial projection matrices (recall that \math{\matA} has
  \math{k} rows). Define \math{\cl{Q}} as this average spectral norm,
  \mld{
    \cl{Q}=\frac{1}{\choose{n}{k}}\sum_{\matA}{\norm{\matP_\matA}_2}.
  }
  Then,
  \mld{
    \cl{Z}
    \ge
    \choose{n}{k}\frac{(1-\cl{Q})^2}{\cl{Q}}.
    \label{eq:proof-leave-A-out-Z}
  }
  Using \r{eq:proof-leave-A-out-Z} in \r{eq:proof-leave-A-out-2} gives
  \mld{
    \Exp[\norm{\matX\ww_\matA\supminus-\yy}^2]
    \le
    \left(1+\frac{k}{n}\frac{\cl{Q}}{(1-\cl{Q})^2}\right)\norm{\matX\ww_*-\yy}^2
    .
    \label{eq:proof-leave-A-out-2.1}
  }
  We need an upper bound on
  \math{\cl{Q}}. Note \math{\norm{\matP_\matA}_2=\norm{\matU_\matA\matU_\matA\transp}_2
    =\norm{\matU_\matA}_2^2.} Therefore,
  \mld{
    \frac{1}{d}\norm{\matU_\matA}_F^2\le\norm{\matP_\matA}_2\le\norm{\matU_\matA}_F^2.}
  Since \math{\norm{\matU_\matA}_F^2=\sum_{j\in\matA}\norm{\uu_j}^2}, we have
  \mld{
    \frac{1}{d\choose{n}{k}}\sum_{\matA}\sum_{j\in\matA}\norm{\uu_j}^2
    \le \cl{Q}
    \le
    \frac{1}{\choose{n}{k}}\sum_{\matA}\sum_{j\in\matA}\norm{\uu_j}^2.
    \label{eq:proof-leave-A-out-3}
  }
  Fix \math{i\in[n]}. The term
  \math{\norm{\uu_j}^2} appears in \math{\choose{n-1}{k-1}} of the
  \math{\matA}s, hence
  \eqar{
    \sum_{\matA}\sum_{j\in\matA}\norm{\uu_j}^2
    &=&\choose{n-1}{k-1}\sum_{i=1}^n
    \norm{\uu_j}^2
    =
    \choose{n-1}{k-1}d,
    \label{eq:proof-leave-A-out-4}
  }
  where the last step uses \math{\sum_{i}
    \norm{\uu_j}^2=d} (orthogonality of \math{\matU}).
  Using \r{eq:proof-leave-A-out-4} in \r{eq:proof-leave-A-out-3} gives
  \mld{
    \frac{k}{n}\le \cl{Q}\le \frac{dk}{n}.
    \label{eq:Q-bound}
  }
  Finally, using the upper bound for \math{\cl{Q}} in \r{eq:Q-bound} in   \r{eq:proof-leave-A-out-2.1} completes
  the proof.
\end{proof}

\noindent
{\bf Comment.} Our analysis of \math{\cl{Q}} in the proof is loose by at most
a factor of \math{d}, which could \math{\sqrt{d}}-factor
increase in the data thrown out.
Indeed, with \math{\cl{Q}=k/n},
\math{k=n\sqrt{d}/(\sqrt{d}+\sqrt{n})} gives approximation
\math{1+d/n}.
Getting tighter bounds on the average squared spectral norm  of
\math{k} rows of an orthogonal \math{n\times d} matrix would have an impact.

\noindent
{\bf Comment.}
There is a big gap between the \math{\Omega(\sqrt{n})}
reduction in label complexity
offered in Theorem~\ref{thm:k-points}
compared to the chaining analysis and lower bound which suggests
that \math{\Omega(n)} is possible.
It is an interesting question whether this gap can be closed.

\noindent
{\bf Comment.} Our inference algorithm on the reduced data is simple
linear regression, the same inference algorithm used on the full data.
One direction
for improving the result is to couple the inference
algorithm to the data thrown out. Specifically reweighting
the left-in data and/or using some form of regularization in the fitting.

\noindent
{\bf Comment.}
The proof is constructive. However getting all sampling
probabilities
exactly is exponential, taking
\math{O(\choose{n}{k}kd\min\{k,d\})} time. We discuss a polynomial
sampling algorithm next.

\section{Polynomial Sampling Algorithm}
\label{sec:sampling}

We wish to exactly sample from the probability distribution
\r{eq:pA} efficiently.
The probabilities
nontrivially depend on  \math{\norm{\matU_\matA}_2} and
the spectral norm itself is hard to deal with.
We give a sampling algorithm based on rejection whose
efficiency depends
on the small leverage scores (a measure of incoherence) but is
otherwise
polynomial. This sampling efficiency can be pre-computed.

Sampling a submatrix using probabilities determined by some nontrivial
property of
the submatrix is generally not easy.
One example is volume sampling~\cite{DRVW2006},
where the probabilities depend on the product of singular values.
One setting where it is easy to sample exactly
is when the probability of a set of rows is the sum of some function over
the rows. Let \math{\matU\transp=[\uu_1,\ldots,\uu_n]} be a matrix with
rows \math{\uu_i}. Let \math{f(\uu)} be a nonnegative function
and define the sampling probability for a set of \math{k}
rows
\math{\matA} as proportional to the sum of \math{f} over the rows
in \math{\matA},
\mld{p_\matA=\frac{1}{\cl{Z}}\sum_{i\in\matA}f(\uu_i),
\label{eq:pA-sum}}
where
\mld{
  \cl{Z}=\sum_{\matA}\sum_{i\in\matA}f(\uu_i)
  =
  \choose{n-1}{k-1}\sum_{i=1}^n f(\uu_i).
  }
For any \math{f}, one can sample exactly using probabilities
\math{p_\matA} in \math{O(n)} time.
\begin{theorem}\label{theorem:sampling}
  Sample one row \math{\uu_i} according to
  the probabilities
  \mld{p_i=\frac{f(\uu_i)}{\sum_{j=1}^nf(\uu_j)}.}
  Sample \math{k-1} rows (without replacement) uniformly
  from the \math{\choose{n-1}{k-1}} possible \math{(k-1)}-subsets
  of the remaining \math{n-1} rows in \math{\matU_{i}\supminus}.
  For any \math{f}, the probability to sample \math{\matA} is given by
  \math{p_\matA} in \r{eq:pA-sum}.
\end{theorem}
\begin{proof}
  Consider the set of rows \math{\matA\transp=[\uu_1,\ldots,\uu_k]}. The same
  argument applies to any other \math{k}
  rows. We compute the probability
  to sample \math{\matA}. Conditioning
  on the first row sampled,
  \eqar{
    \Prob[\matA]
    &=&
    \sum_{i=1}^k
    \Prob[\text{\math{\uu_i} is the first row sampled}]
    \Prob[\text{\math{\matA_i\supminus} are the remaining rows sampled}\given\uu_i]
    \\
    &=&
    \sum_{i=1}^k p_i\times\frac{1}{\choose{n-1}{k-1}}
    \\
    &=&
    \frac{1}{\choose{n-1}{k-1}\sum_{j=1}^nf(\uu_j)}\sum_{i=1}^kf(\uu_i)
    \\
    &=&
    p_\matA,
  }
  where the last step follows from the definitions of
  \math{p_\matA} and \math{\cl{Z}}.
\end{proof}
{\bf Comment.} Sampling using Frobenius norm probabilities,
\math{p_\matA\propto\sum_{i\in\matA}\norm{\uu_i}^2} fits the assumptions of the
theorem with \math{f(\uu)=\norm{\uu}^2}. Sampling using inverse sum of
leverage scores also fits, where \math{f(\uu)=1/\norm{\uu}^2} in which case
\math{p_\matA\propto\sum_{i\in\matA}1/\norm{\uu_i}^2}.

We use rejection to sample \math{\matA} according to the
probabilities in \r{eq:pA}. Here is the algorithm.
\begin{center}
 \fbox{
     \begin{minipage}{0.875\linewidth}
    \begin{algorithmic}[1]
      \State Sampling \math{\matA} using probabilities \math{p_\matA} in
      \r{eq:pA}.
  \Repeat
  \State Sample \math{\matA} using Theorem~\ref{theorem:sampling}
  and the probabilities \math{q_\matA}
  given by \math{f(\uu)=1/\norm{\uu}^2},
  \mld{
    q_\matA=\frac{1}{\choose{n-1}{k-1}\sum_{j=1}^n 1/\norm{\uu_j}^2}
    \sum_{i\in\matA}{1}/{\norm{\uu_i}^2}.
  }
  \State Accept \math{\matA} with probability
  \mld{\theta_\matA=
    \frac{(1-\norm{\matU_\matA}_2^2)^2/\norm{\matU_\matA}_2^2}{\frac{d}{k^2}
      \sum_{i\in\matA}1/\norm{\uu_i}^2}.
  \label{eq:theta-A}}
  \Until{\math{\matA} is accepted.}
    \end{algorithmic}
    \end{minipage}
      }
\end{center}
First, to show that the rejection sampling is valid, we need that
\math{\theta_\matA\le 1}. Indeed this is the case. We prove it as follows.
Using \math{\norm{\matU_\matA}_2^2\ge\norm{\matU_\matA}_F^2/d} and
\math{\norm{\matU_\matA}_F^2=\sum_{i\in\matA}\norm{\uu_i}^2} gives
\mld{
  \frac{(1-\norm{\matU_\matA}_2^2)^2}{\norm{\matU_\matA}_2^2}
  \le
  \frac{(1-\frac{1}{d}\sum_{i\in\matA}\norm{\uu_i}^2)^2}{\frac{1}{d}
    \sum_{i\in\matA}\norm{\uu_i}^2}
  \le
  \frac{d}{\sum_{i\in\matA}\norm{\uu_i}^2}.
  \label{eq:numerator-theta-A}
}
We use a convexity argument to bound
\math{d/\sum_{i\in\matA}\norm{\uu_i}^2},
\mld{
  \frac{d}{\sum_{i\in\matA}\norm{\uu_i}^2}
  =\frac{d}{k}\cdot\frac{1}{\frac{1}{k}\sum_{i\in\matA}\norm{\uu_i}^2}
  \le
  \frac{d}{k}\cdot\frac{1}{k}\sum_{i\in\matA}1/\norm{\uu_i}^2
  =
  \frac{d}{k^2}\sum_{i\in\matA}1/\norm{\uu_i}^2.
  \label{eq:denominator-theta-A}
}
Combining \r{eq:numerator-theta-A} and \r{eq:denominator-theta-A} in
\r{eq:theta-A} establishes that \math{\theta_\matA\le 1}, so the
rejection sampling is valid. We now show that the
probability distribution of \math{\matA} conditioned on it being
accepted is as desired in \r{eq:pA}. Indeed,
\eqar{
  \Prob[\matA\given\text{accept}]
  &=&
  \frac{\Prob[\matA\cap\text{accept}]}{\Prob[\text{accept}]}
  \\
  &=&
  \frac{q_\matA\theta_\matA}{\sum_\matA q_\matA\theta_\matA}.
}
For the probability to accept, we have
\eqar{
  \sum_{\matA}q_\matA\theta_\matA
  &=&
  \frac{1}{\frac{d}{k^2}\choose{n-1}{k-1}\sum_{j=1}^n 1/\norm{\uu_j}^2}
  \sum_{\matA}\frac{(1-\norm{\matU_\matA}_2^2)^2}{\norm{\matU_\matA}_2^2}
  \\
  &=&
  \frac{\cl{Z}}{\frac{d}{k^2}\choose{n-1}{k-1}\sum_{j=1}^n 1/\norm{\uu_j}^2}
  .
}
Dividing \math{q_\matA\theta_A} by the above gives
\mld{
  \Prob[\matA\given\text{accept}]
  =
  \frac{1}{\cl{Z}}\frac{(1-\norm{\matU_\matA}_2^2)^2}{\norm{\matU_\matA}_2^2},
}
as desired. The expected number of
trials to accept \math{\matA} is given by
\math{1/\Prob[\text{accept}]}. The cost of a
trial is the time to generate
a sample according to the probabilities \math{q_\matA},
which is \math{O(n)},
plus the time to compute \math{\theta_A} which is
\math{O(kd\min\{k,d\})}.
Hence, the expected runtime is
\mld{
  \text{runtime}=\frac{O(n+kd\min\{k,d\})}{\Prob[\text{accept}]}.
}
We need a lower bound on \math{\Prob[\text{accept}]}.
For the input matrix \math{\matX},
define a measure of coherence \math{\mu},
\mld{
  \mu=\frac{1}{n}\sum_{i=1}^n\frac{1}{\norm{\uu_i}^2},
}
This measure of coherence is the average of the reciprocals of the
leverage scores. If the leverage scores are uniform, then
\math{\mu=n/d}. In general, \math{\mu\ge n/d} by convexity.
The larger \math{\mu}, the less uniform the
leverage scores.
The coherence \math{\mu} captures how many of the leverage scores are small.
We get a lower bound for
\math{\Prob[\text{accept}]} in terms of \math{\mu} as follows.
\eqar{
  \Prob[\text{accept}]
  &=&
  \frac{k^2}{nd\mu\choose{n-1}{k-1}}
    \sum_{\matA}\frac{(1-\norm{\matU_\matA}_2^2)^2}{\norm{\matU_\matA}_2^2}
}
To get a lower bound for the sum over \math{\matA}, we use
\math{\norm{\matU_A}^2_2\le\norm{\matU_A}^2_F},
\eqar{
  \sum_{\matA}\frac{(1-\norm{\matU_\matA}_2^2)^2}{\norm{\matU_\matA}_2^2}
  &\ge&
  \sum_{\matA}\frac{(1-\norm{\matU_\matA}_F^2)^2}{\norm{\matU_\matA}_F^2}
  \nonumber
  \\
  &=&
  \sum_{\matA}{\norm{\matU_\matA}_F^2}+
  \sum_{\matA}\frac{1}{\norm{\matU_\matA}_F^2}
  -2\choose{n}{k}
  \nonumber
  \\
  &=&
  \choose{n-1}{k-1}d+\sum_{\matA}\frac{1}{\norm{\matU_\matA}_F^2}
  -2\choose{n}{k}
  \nonumber
  \\
  &\ge&
  \choose{n-1}{k-1}d.\label{eq:Z-lower-bound}
}
where the last step follows from Lemma~\ref{lemma-positive-residual} which
states that
\math{\sum_{\matA}1/\norm{\matU_\matA}_F^2>2\choose{n}{k}} when
\math{n\ge 8dk}. Since \math{k\in\Theta(\sqrt{n})}, this means
\math{n\in\Omega(d^2)}.
\mld{
  \Prob[\text{accept}]\ge\frac{k^2}{n\mu}.
}
Since \math{k^2\in\Theta(n)}, this means that
\math{\text{runtime}\in O(\mu(n+kd\min\{k,d\}))}.
Since \math{\mu} is typically of order \math{n/d}, the runtime is in
\math{O(n^2/d)}, a polynomial runtime. We now prove
the last step in~\r{eq:Z-lower-bound}.
\begin{lemma}\label{lemma-positive-residual}
  For \math{n\ge 8dk},
  \math{\displaystyle
    \sum_{\matA}\frac{1}{\norm{\matU_\matA}_F^2}-2\choose{n}{k}\ge0.
  }
\end{lemma}
\begin{proof}
  Define \math{\cl{A}_{\text{bad}}} as the set of bad \math{\matA}s for which
  \math{\norm{\matU_\matA}_F^2\ge 1/4}. Then,
  \mld{
    \choose{n-1}{k-1}d=\sum_{\matA}\norm{\matU_\matA}_F^2
    =\sum_{\matA\in\cl{A}_{\text{bad}}}\norm{\matU_\matA}_F^2
    +
    \sum_{\matA\not\in\cl{A}_{\text{bad}}}\norm{\matU_\matA}_F^2
    \ge\frac{|\cl{A}_{\text{bad}}|}{4}.
  }
  This means
  \mld{|\cl{A}_{\text{bad}}|\le 4d\choose{n-1}{k-1}.\label{eq:bound-Abad}}
  We therefore
  have that
  \eqar{
    \sum_{\matA}\frac{1}{\norm{\matU_\matA}_F^2}
    &=&
    \sum_{\matA\in\cl{A}_{\text{bad}}}\frac{1}{\norm{\matU_\matA}_F^2}
    +
    \sum_{\matA\not\in\cl{A}_{\text{bad}}}\frac{1}{\norm{\matU_\matA}_F^2}
    \nonumber
    \\
    &\ge&
    \sum_{\matA\not\in\cl{A}_{\text{bad}}}\frac{1}{\norm{\matU_\matA}_F^2}
    \nonumber
    \\
    &\ge&
    \left(\choose{n}{k}-|\cl{A}_{\text{bad}}|\right)\times 4
  \nonumber
  \\
  &\ge&
  4\choose{n}{k}-16d\choose{n-1}{k-1},
  }
  where the last step uses \r{eq:bound-Abad}.
  Subtracting \math{2\choose{n}{k}} gives
  \mld{
    \sum_{\matA}\frac{1}{\norm{\matU_\matA}_F^2}-2\choose{n}{k}
    \ge 2\choose{n}{k}-16d\choose{n-1}{k-1}
    =2\choose{n-1}{k-1}\left(\frac{n}{k}-8d\right).
  }
  The lemma follows from the assumption \math{n\ge 8dk}.
\end{proof}

{\small
\bibliographystyle{natbib}
\bibliography{mypapers,masterbib,Local} 
}
\clearpage

\appendix
\section{Consistent \math{\ell_2} Regression}
\label{sec:app:consistent}

We prove Theorem~\ref{theorem:consistent} for consistent regression,
including the version with fast preprocessing to get approximate
leverage scores and approximate preconditioners. As far as we know
Lemma~\ref{lemma:precond}
in the general form stated is new, and may be of independent
interest.

\subsection{Mathematical Preliminaries}
\label{sec:prelim}

\subsubsection*{\math{\ell_2}-Embeddings}

The matrix \math{\Pi} is an \math{\epsilon}-JLT for an orthogonal matrix
\math{\matU} if
\mld{
  \norm{\matI-\matU\transp\Pi\transp\Pi\matU}_2\le \epsilon.
}
That is, if \math{\Pi\matU} is almost orthogonal. The next lemma summarizes
some of the consequences of an \math{\epsilon}-JLT, which are well known, see
for example
\cite{DMMS07_FastL2_NM10}.
\begin{lemma}\label{lemma:embedding-properties}
Let \math{\matU\in\R^{n,d}} be orthogonal and
\math{\Pi\in\R^{n\times r}}.
Suppose \math{\norm{\matI-\matU\transp\Pi\transp\Pi\matU}\le\epsilon<1}. Then,
\begin{enumerate}
\item \math{|1-\sigma_i^2(\Pi\matU)|\le\epsilon} and
\math{\rank(\Pi\matU)=d}.
\item Let \math{\Pi\matU=\tilde\matU\tilde\Sigma\tilde\matV\transp}.
Then,
\math{\norm{\tilde\Sigma-\tilde\Sigma^{-1}}\le \epsilon/\sqrt{1-\epsilon}}
and \math{\norm{\matI-\tilde\Sigma^{-2}}\le\epsilon/(1-\epsilon)}.
\item
\math{\norm{(\Pi\matU)^{\dagger}-(\Pi\matU)\transp}\le\epsilon/\sqrt{1-\epsilon}}.

\item Let \math{\matA=\matU\Sigma\matV\transp}, where
\math{\Sigma} is positive diagonal and \math{\matV} is orthogonal. Then
\math{(\Pi\matA)^{\dagger}=\matV\Sigma^{-1}(\Pi\matU)^{\dagger}}.
\item \math{\norm{\matI-(\Pi\matU)^{\dagger}(\Pi\matU)^{\dagger\textsc{t}}}\le \epsilon/(1-\epsilon).}
\end{enumerate}
\end{lemma}
\begin{proof}
Part 1 is immediate and implies Part 2 which implies Part 3
because
\math{(\Pi\transp\matU)^{\dagger}=
\tilde\matV\tilde\Sigma^{-1}\tilde\matU\transp}, so
\mand{\norm{(\Pi\transp\matU)^{\dagger}-(\Pi\transp\matU)\transp}
=
\norm{\tilde\matV\tilde\Sigma^{-1}\tilde\matU\transp-\tilde\matV\tilde\Sigma\tilde\matU\transp}
\le
\norm{\tilde\Sigma^{-1}-\tilde\Sigma}.}
Part 4 follows from properties of the pseudo-inveerse. For Part 5, using
\math{\Pi\transp\matU=\tilde\matU\tilde\Sigma\tilde\matV\transp} and
\math{\tilde\matV\tilde\matV\transp=\matI} (since
\math{\Pi\transp\matU} has full rank, so \math{\tilde\matV} is a square
orthogonal matrix),
\mand{\norm{\matI-(\Pi\transp\matU)^{\dagger}(\Pi\transp\matU)^{\dagger\textsc{t}}}
  =
  \norm{\tilde\matV\tilde\matV\transp-\tilde\matV\tilde\Sigma^{-2}\tilde\matV\transp}=\norm{\matI-\Sigma^{-2}}\le\epsilon/(1-\epsilon),}
where the last step follows from Part 2. 
\end{proof}

The next result shows that an \math{\epsilon}-JLT can be used
find a preconditioner.
\begin{lemma}\label{lemma:precond}
  Let \math{\matA=\matU\matSig\matV\transp} be \math{n\times d}, having full rank \math{d}.
  Let \math{\Pi\matA=\matQ\matR}, where \math{\matQ} is any orthogonal
  matrix and \math{\matR} is invertible. 
  Then \math{\sigma_i^2(\matA\matR^{-1})=1/\sigma_{d+1-i}^2(\Pi\matU)}.
\end{lemma}
\begin{proof}
  Consider the SVD of \math{\matA\matR^{-1}
    {\matR^{-1}}\transp\matA\transp}.
  Using
  \math{R^{-1}=(\matQ\transp\Pi\matA)^{-1}=(\Pi\matA)^\dagger\matQ},
  \eqar{
   \matA\matR^{-1}{\matR^{-1}}\transp\matA\transp
   &=&
   \matA(\Pi\matA)^\dagger\matQ\matQ\transp{(\Pi\matA)^\dagger}\transp\matA\transp
   \\
   &=&
   \matU\matSig\matV\transp
   \matV\Sigma^{-1}(\Pi\matU)^{\dagger}
   \matQ\matQ\transp
       {(\Pi\matU)^{\dagger}}\transp\Sigma^{-1}\matV\transp
       \matV\matSig\matU\transp
       \\
       &=&
       \matU(\Pi\matU)^{\dagger}
   \matQ\matQ\transp
       {(\Pi\matU)^{\dagger}}\transp\matU\transp
       \\
       &=&
   \matU(\Pi\matU)^{\dagger}
                  {(\Pi\matU)^{\dagger}}\transp\matU\transp.
                  \label{eq:variational}
   }
  In the last step \math{\matQ\matQ\transp} projects
  onto the column space of \math{\Pi\matU}, hence
  \math{\matQ\matQ\transp
    {(\Pi\matU)^{\dagger}}\transp={(\Pi\matU)^{\dagger}}\transp}.
  Using \math{\Pi\matU=\matU_{\Pi\matU}\matSig_{\Pi\matU}\matV_{\Pi\matU}\transp}, we get
  \math{(\Pi\matU)^{\dagger}
                  {(\Pi\matU)^{\dagger}}\transp=\matV_{\Pi\matU}\matSig_{\Pi\matU}^{-2}\matV_{\Pi\matU}\transp}, hence
  \mld{
    \matA\matR^{-1}{\matR^{-1}}\transp\matA\transp
    =
    \matU\matV_{\Pi\matU}\matSig_{\Pi\matU}^{-2}\matV_{\Pi\matU}\transp
    \matU\transp.
  }
  Since \math{\matU\matV_{\Pi\matU}} is orthogonal, we have
  constructed the SVD of \math{\matA\matR^{-1}{\matR^{-1}}\transp\matA\transp} and so up to a rotation of the row space, we can
  write down the SVD of \math{\matA\matR^{-1}}.
  For some orthogonal \math{d\times d} matrix \math{\matZ},
  \mld{
    \matA\matR^{-1}=
    \matU\matV_{\Pi\matU}\matSig_{\Pi\matU}^{-1}\matZ\transp
  }
  That is, the singular values of
  \math{\matA\matR^{-1}} are the inverses of the singular values of
  \math{\Pi\matU}.
\end{proof}
A useful corollary of Lemma~\ref{lemma:precond}
was observed in \cite{RT08}, namely that
\math{\matA\matR^{-1}} and \math{\Pi\matU} have the same condition
number. Indeed,
\mld{
  \kappa(\matA\matR^{-1})
  =
  \frac{\sigma_1^2(\matA\matR^{-1})}{\sigma_d^2(\matA\matR^{-1})}
  =
  \frac{1/\sigma_d^2(\Pi\matU)}{1/\sigma_1^2(\Pi\matU)}
  =
  \frac{\sigma_1^2(\Pi\matU)}{\sigma_d^2(\Pi\matU)}
  =
  \kappa(\Pi\matU).
}

We make heavy use of \math{\epsilon}-JLTs which can be constructed
and applied efficiently.
All constructions use some version of a Johnson-Lindenstrauss Transform.
A finite collection of points can be embeded into lower dimension while
preserving norms to relative error and inner products to additive error.
\begin{lemma}[JLT, \cite{JL1984,Ach03_JRNL}]\label{lemma:JLT}
Let \math{\Pi\in\R^{d\times r}} be a matrix of independent
random signs scaled by
\math{1/\sqrt{r}}.
For \math{n} points \math{\xx_1,\ldots,\xx_n\in
\R^{d}}, let \math{\zz_i=\Pi\transp\xx_i}.
For \math{0<\epsilon<1} and \math{\beta>0}, if
\mld{r\ge\frac{8 +4\beta}{\epsilon^2-2\epsilon^3/3}\ln (n+1),
\label{eq:JLT-r}}
then, with probability at least \math{1-n^{-\beta}},
for all \math{i,j\in[1,n]}:
\eqan{
&(1-\epsilon)\norm{\xx_i-\xx_j}^2\le\norm{\zz_i-\zz_j}^2\le
(1+\epsilon)\norm{\xx_i-\xx_j}^2&\\
&
|\zz_i\transp\zz_j-\xx_i\transp\xx_j|\le
\epsilon(\norm{\xx_i}^2+\norm{\xx_j}^2).
}
\end{lemma}
The \math{\ln(n+1)} comes from adding \math{\bm0} to the points
which preserves all norms. Also, 
\math{\xx_i\transp\xx_j
=\frac12(\norm{\xx_i}^2+\norm{\xx_j}^2-\norm{\xx_i-\xx_j}^2)},
hence inner products are preserved to additive error
\math{\epsilon(\norm{\xx_i}^2+\norm{\xx_j}^2)}.

Using this result, we can get \math{\ell_2}-subspace embeddings using a
variety of constructions. The one we will use is an
oblivious fast-Hadamard subspace embedding known as the
Subsampled Random Hadamard Transform (SRHT). A result
from~\cite{T2011} which refines earlier results \cite{AR2014} is given
in next lemma which is an  application of Lemmas~3.3 and~3.4 in~\cite{T2011}.
The runtime in the lemma 
is established in~\cite{AL2013}.
\begin{lemma}[{\cite[Lemmas 3.3 and 3.4]{T2011}}]\label{lemma:SRHT-tropp}
  Fix \math{0<\varepsilon\le\frac12}. Let
  \math{\matU\in\R^{n\times d}}
  be orthogonal and \math{\Pi_\textsc{h}\in\R^{r\times n}} an SRHT
  with embedding dimension \math{r}
  satisfying:
  \mld{
    r\ge \frac{12}{5\varepsilon^2}\left(\sqrt{d}+\sqrt{8\ln(3n/\gamma)}\right)^2\ln d
    \ \in\ 
    O\left(\frac{\ln d}{\varepsilon^2}(d+\ln(n/\gamma))\right).
  \label{eq:FJLT-r}}
  Then, with probability at least \math{1-\gamma},
  \math{\Pi_\textsc{h}} is an \math{\epsilon}-JLT for \math{\matU},
  \mld{
    \norm{\matI-\matU\transp\Pi_\textsc{h}\transp\Pi_\textsc{h}\matU}_2\le\varepsilon.
  }
  Further, the product \math{\Pi_\textsc{h}\matA} can be computed in
  time \math{O(nd\ln r)} for any matrix \math{\matA\in\R^{n\times d}}.
\end{lemma}
Note that the SHRT \math{\Pi_{\textsc{h}}} is constructed obliviously of
\math{\matU}. We use the notation \math{\epsilon}-FJLT for such
JLTs that can be applied fast, to distinguish it from the regular
\math{\epsilon}-JLT in Lemma~\ref{lemma:JLT}.

\subsection{Randomized Pre-Conditioned Kaczmarz} 
\label{section:algo}

We first consider the consistent case, that is, there exists
\math{\ww_*} for which
\mld{\matX\ww_*=\yy.}
Equivalently, for any invertible matrix \math{\matD}, we can solve
\math{\matD\matX\ww=\matD\yy}.
The idea is to sample not using row-norms of \math{\matX}, but sample using
row-norms of some orthogonal basis for the column-space of \math{\matX}.
Let
\math{
  \matX=\matU\matSig\matV\transp.
}
If we solve \math{\vv_*} satisfying \math{\matU\vv_*=\yy}, then we recover
\math{\ww_*} using
\mld{
  \ww_*=\matV\matSig^{-1}\vv_*.
}
\math{\matU} is well conditioned so randomized Kaczmarz
has a label complexity \math{d\log(1/\epsilon)} and runtime
\math{d^2\log(1/\epsilon)}. The problem is getting
\math{U} is expensive. However, to identify the rows we need, we don't need
\math{\matU}, we just need the leverage scores. And a fast constant factor
approximation to the leverage scores will do. This can be accomplished
via two JLT's, one for the column-space and one for the row space,
\cite{malik186}.

Let \math{\Pi_1\in\R^{r_1 \times n}} with \math{r_1\in O(d\ln d)}
be an \math{\epsilon}-FJLT for \math{\matU} satisfying
\mld{
  \norm{\matI-(\Pi_1\matU)\transp\Pi_1\matU}\le \frac12.
}
The matrix \math{\matX_1=\Pi_1\matX}
can be computed in \math{O(nd\log r_1)=O(nd\log d)}.
There are \math{\matQ\in\R^{r_1\times d},\matT\in\R^{d\times d},\matP\in\R^{d\times d}}
such that
\mld{\Pi_1\matX=\matQ\matT\matP=\matQ\matR,}
where \math{\matQ} is orthogonal, \math{\matT} is upper triangular,
\math{\matP} is a permutation matrix and \math{\matR=\matT\matP}.
The reason writing \math{\matR} as the product \math{\matT\matP} is that
one can apply
\math{\matR^{-1}} or \math{(\matR^{-1})\transp} to any vector
\math{\xx} in \math{O(d^2)} without explictly computing \math{\matR^{-1}}.
This is because a permutation matrix is orthogonal so
\math{\matR^{-1}=\matP\transp\matT^{-1}}
and applying the inverse of an upper triangular \math{d\times d} matrix can
be done efficiently in \math{O(d^2)} \cite{GV96}.
Also note,
\mld{
  R^{-1}=(\matQ\transp\Pi_1\matX)^{-1}
  =
  (\Pi_1\matX)^{\dagger}(\matQ\transp)^\dagger
  =
  (\Pi_1\matX)^{\dagger}\matQ
  .
}
We now estimate the row-norms of \math{\matX\matR^{-1}} to relative error
by applying a
\math{JLT} to the rows. Specifically let \math{\Pi_2\in \R^{d\times r_2}} be
a JLT with \math{r_2\in O(\log n)}, satisfying
\mld{
    \frac12\norm{\ee_i\transp \matX\matR^{-1}}^2
    \le
    \norm{\ee_i\transp \matX\matR^{-1}\mat\Pi_2}^2
  \le
  \frac32
  \norm{\ee_i\transp \matX\matR^{-1}}^2.
  \label{eq:alg:JLT}
}
That is, we can estimate all the row-norms in \math{ \matX\matR^{-1}}
to within constant relative error by using the
row-norms in
\math{\matX\matR^{-1}\mat\Pi_2}.
We can compute \math{\matX\matR^{-1}\mat\Pi_2} in runtime
\math{O(ndr_2)=O(nd\log n)}. Hence, in \math{O(nd\log n+d^2\log n)}
we can compute
\math{\hat\ell_1,\ldots,\hat\ell_n}, where
\mld{
  \hat\ell_i=\norm{\ee_i\transp \matX\matR^{-1}\mat\Pi_2}^2.
\label{eq:lev-approx}
}
The reason we introduce these quantities \math{\hat\ell_i} is that they
approximate to within relative error the leverage scores of
\math{\matX}. Consider \math{\norm{\ee_i\transp \matX\matR^{-1}}^2},
\eqar{
  \norm{\ee_i\transp \matX\matR^{-1}}^2
  &=&
  \norm{\ee_i\transp \matX(\Pi_1\matX)^{\dagger}\matQ}^2
  \\
  &=&
  \ee_i\transp \matX(\Pi_1\matX)^{\dagger}\matQ\matQ\transp((\Pi_1\matX)^{\dagger})\transp
  \matX\transp\ee_i
  \\
  &=&
  \ee_i\transp \matX(\Pi_1\matX)^{\dagger}((\Pi_1\matX)^{\dagger})\transp
  \matX\transp\ee_i
  \\
  &=&
  \norm{\ee_i\transp \matX(\Pi_1\matX)^{\dagger}}^2.
  \label{eq:lev-JLT}
}
Since \math{\Pi_1\matX=\Pi_1\matU\matSig\matV\transp},
\mld{\matX(\Pi_1\matX)^\dagger=\matU\matSig\matV\transp\matV\matSig^{-1}(\Pi_1\matU)^\dagger=\matU(\Pi_1\matU)^\dagger.}
Therefore,
\eqar{
  \norm{\ee_i\transp \matX\matR^{-1}}^2
  &=&
  \ee_i\transp \matU(\Pi_1\matU)^\dagger{(\Pi_1\matU)^\dagger}\transp
  \matU\transp\ee_i
  \label{eq:approx-leverage}
}
By Part 5 of Lemma~\ref{lemma:embedding-properties},
\math{(\Pi_1\matU)^\dagger{(\Pi_1\matU)^\dagger}\transp\approx\matI_d}.
That means the row norms of \math{\matX\matR^{-1}} are approximately the
statistical leverage scores \math{\norm{\ee_i\matU}^2}. In fact,
the estimator in
\r{eq:lev-JLT} is exactly the one analyzed in
\cite{malik186} to obtain
an efficient approximation to the leverage scores.
This means
\math{\matX\matR^{-1}} is well conditioned.

We can now state the randomized Kaczmarz algorithm.
In a nutshell, sample rows of \math{\matX} using probabilities
proportional to \math{\hat\ell_i}, as opposed to
proportional to \math{\norm{\xx_i}^2}, and perform the projective
update using the preconditioned row \math{\xx_i\matR^{-1}}. By
\r{eq:lev-JLT},
\math{\hat\ell_i\approx\norm{\xx_i\matR^{-1}}^2}. In effect,
we are using Kaczmarz to solve the system \math{\matX\matR^{-1}\vv=\yy}
but instead of sampling using row-norms, we are sampling using
approximate row-norms. Since
\math{\matX\matR^{-1}} is well conditioned, the algorithm
will have exponential convergence independent of the input-conditioning.

\begin{algorithmic}[1]
  \State Construct the matrix \math{\matR} as described above
  using an \math{\epsilon}-FJLT for \math{\R^{n\times d}}.
  \State Compute the leverage scores \math{\hat\ell_i}
  as described in \r{eq:lev-approx}.
  \State Compute the cumulative probabilities
  \math{F_k=\sum_{i=1}^k\hat\ell_i/\sum_{j=1}^n\hat\ell_j}.
  \State Initialize the weights to \math{\vv_0=\bm0}.
  \For{\math{t=1,\ldots,K}}
  \State Independently sample (with replacement) an index \math{j\in [n]}
 using the probabilities \math{\hat\ell_i/\sum_{i}\hat\ell_i}.
 \State Let \math{\qq=\xx_{j}\transp \matR^{-1}=\ee_j\transp\matX\matR^{-1}}
 and let \math{s=y_{j}}.
  Perform the projective weight update,
  \mld{
    \vv_t=\vv_{t-1}-\frac{\qq(\qq\transp\vv_{t-1}-s)}{\norm{q}^2}
    \label{eq:kaczmarz-update}
  }
  \EndFor
  \State{\bf return} \math{\ww_K=\matR^{-1}\vv_K}.
\end{algorithmic}

Each step's runtime is as follows.

  \begin{algorithmic}[1]
    \State \math{O(nd\log d)} to compute \math{\Pi_1\matX} and then
    \math{O(d^3\log d)} to get \math{\mat{R}} from a QR-factorization.
    \State \math{O(nd\log n+d^2\log n)} to get \math{\hat\ell_i}.
    \State \math{O(n)} to get all cumulative probabilities
  \State \math{O(d)}.
  \For{\math{t=1,\ldots,K}}
  \State \math{O(\log n)} to sample once using binary search, so
  \math{O(K\log n)} in total.
    \State Applying \math{R^{-1}} is \math{O(d^2)} and the
  projective weight update is \math{O(d)}, so \math{O(Kd^2)} in total.
  \EndFor
  \State Applying \math{R^{-1}} is \math{O(d^2)}
  \end{algorithmic}
  
  Adding all of the above gives
  \math{O(nd\log n+d^3\log d)} preprocessing time in steps 1-3 and
  \math{O(K(\log n+d^2))} for running the Kaczmarz iterations.
  The label complexity is \math{K}.

  {\bf Note.} In step 6, \math{\matR^{-1}} has to be applied to each of
  \math{\xx_{j_1},\xx_{j_2},\ldots,\xx_{j_K}}, which can be pre-sampled
  ahead of time since the sampling is independent according to fixed
  probabilities. This requires solving an upper
  triangular system with multiple right hand sides,
  \mld{
    \matR\ [\qq_1,\qq_2,\ldots,\qq_K]=[\xx_{j_1},\xx_{j_2},\ldots,\xx_{j_K}].
  }

\subsection{Input-Independent Exponential Convergence of Kaczmarz} 

We now state and prove the main result for the consistent
case where we use the Algorithm described in the previous section.
In a nutshell, the convergence is
exponential
and does not depend on the conditioning of the input matrix \math{\matX}.
We make some simplifying assumptions, which in some cases
are almost vacuous. Set the failure probability to
\math{\gamma=1/n}, assume \math{n\ge 3} and \math{\ln n\le d}. Setting
\math{\epsilon=1/2} in \r{eq:FJLT-r} and simplifying, an SRHT with
\math{r_1\ge 48 d\ln d} is an \math{(1/2)}-FJLT for any
orthogonal \math{\matU\in\R^{n\times d}}, with probability at least
\math{1-1/n}. Similarly, with
\math{\beta=1, \epsilon=1/2} in \r{eq:JLT-r}, Lemma~\ref{lemma:JLT}
with \math{r_2\ge 72\ln(1+n)} gives a norm preserving \math{(1/2)}-JLT
for any \math{n} points in \math{\R^d}. 
\begin{theorem}\label{thm:kaczmarz}
  With probability  at least \math{1-2/n}
  (w.r.t. the random choice of \math{\Pi_1} and
  \math{\Pi_2}), the 
  Algorithm in the previous section has the following properties.
  \begin{enumerate}[nolistsep]
  \item The label complexity is \math{K}.
  \item The preprocessing time is in \math{O(nd\log n)}.
  \item The time to run the algorithm for a single right hand side \math{\yy}
    is in \math{O(K(\log n+d^2))}.
  \item The quality of approximation for \math{\ww_t=\matR^{-1}\vv_t}, where \math{t\in[K]}
    is determined by 
    \mld{
      \Exp[\norm{\vv_t-\vv_*}^2]
      \le
      \left(1-\frac{1}{9d}\right)^t\norm{\vv_*}^2.
    }
  \end{enumerate}
  Where \math{\vv_*=\matR\ww_*} and the expectation is over the \math{K} random rows used in the
  projective
  updates.
\end{theorem}
\begin{proof}
  In the weight update \r{eq:kaczmarz-update},
  use \math{s=\qq\transp\vv_*} and
  subtract \math{\vv_*} from both sides giving,
  \eqar{
    \vv_t-\vv_*
    &=&
    \left(\matI_d-\frac{\qq\qq\transp}{\norm{q}^2}\right)(\vv_{t-1}-\vv_*)
  }
  Take the norm-squared of both sides and using
  \math{(\matI_d-\qq\qq\transp/\norm{\qq}^2)^2=(\matI_d-\qq\qq\transp/\norm{\qq}^2)}
  gives
  \eqar{
    \norm{\vv_t-\vv_*}^2
    &=&
    (\vv_{t-1}-\vv_*)\transp
    \left(\matI_d-\frac{\qq_i\qq_i\transp}{\norm{q_i}^2}\right)
    (\vv_{t-1}-\vv_*)
    \\
    &=&
    \norm{\vv_{t-1}-\vv_*}^2-(\vv_{t-1}-\vv_*)\transp\frac{\qq\qq\transp}{\norm{q}^2}(\vv_{t-1}-\vv_*)
  }
  Taking the expectation of both sides,
  \eqar{
    \Exp[\norm{\vv_t-\vv_*}^2]
    &=&
    \norm{\vv_{t-1}-\vv_*}^2-\sum_{i=1}^n\frac{\hat\ell_i}{\sum_{j=1}^n\hat\ell_j}(\vv_{t-1}-\vv_*)\transp\frac{\qq\qq\transp}{\norm{q}^2}(\vv_{t-1}-\vv_*).
    \label{eq:exp1}
  }
  Using \r{eq:alg:JLT} and noting that
  \math{\qq_i=\ee_i\transp\matX\matR^{-1}}, we have that
  \mld{
    \frac12\norm{\qq_i}^2
    \le
    \hat\ell_i
  \le
  \frac32
  \norm{\qq_i}^2.
  }
  Therefore \math{\hat\ell_i/\sum_j\hat\ell_j\ge
    \frac12\norm{\qq_i}^2/ \frac32\sum_j
    \norm{\qq_j}^2}. Using \r{eq:exp1} and
  \math{\sum_j
    \norm{\qq_j}^2=\norm{\matX\matR^{-1}}^2_F} gives 
  \eqar{
    \Exp[\norm{\vv_t-\vv_*}^2]
    &=&
    \norm{\vv_{t-1}-\vv_*}^2-\frac{1}{3\norm{\matX\matR^{-1}}_F^2}(\vv_{t-1}-\vv_*)\transp\left(\sum_{i=1}^n\qq_i\qq_i\transp\right)(\vv_{t-1}-\vv_*)
    \\
    &{\buildrel (a) \over=}&
    \norm{\vv_{t-1}-\vv_*}^2-\frac{1}{3\norm{\matX\matR^{-1}}_F^2}
    \norm{\matX\matR^{-1}(\vv_{t-1}-\vv_*)}^2
  \\
  &\le&
  \left(
  1-\frac{\sigma_{d}^2(\matX\matR^{-1})}{3\norm{\matX\matR^{-1}}_F^2}
  \right)
  \norm{\vv_{t-1}-\vv_*}^2.
    \label{eq:exp2}
  }
  In (a) we used
  \math{\sum_i\qq_i\qq_i\transp={\matR^{-1}}\transp\matX\transp\matX\matR^{-1}}.
  To complete the proof we need to analyze the singular values of
  \math{\matX\matR^{-1}}.
  By Lemma~\ref{lemma:precond}, \math{\sigma_i^2(\matX\matR^{-1})
    =1/\sigma_{d+1-i}^2(\Pi_1\matU)}. Since
  \math{\Pi_1} is a (1/2)-FJLT for \math{\matU}, by part 1 in Lemma
  \ref{lemma:embedding-properties}
  \math{2/3\le 1/\sigma_i^2(\Pi_1\matU)\le 2}.
  Hence,
  \mld{
    \frac{\sigma_d^2(\matX\matR^{-1})}{\norm{\matX\matR^{-1}}^2_F}
    =
    \frac{1/\sigma_1^2(\Pi_1\matU)}{\sum_{i\in[d]}1/\sigma_i^2(\Pi_1\matU)}
    \ge
    \frac{2/3}{\sum_{i\in[d]}2}
    =
    \frac{1}{3d}.\label{eq:condition}
  }
  Using \r{eq:condition} in \r{eq:exp2} gives
  \mld{\Exp[\norm{\vv_t-\vv_*}^2]\le
 \left(
  1-\frac{1}{9d}
  \right)
  \norm{\vv_{t-1}-\vv_*}^2.
  \label{eq:app:convergence}
  }
  Since the rows are sampled independently, using iterated
  expectation gives the final result.
\end{proof}
Note that the convergence of \math{\vv_t} is exponential and input independent,
but the condition number does appear in the number of iterations needed
to get the error below a threshold. Since \math{\vv=\matR\ww},
\mld{
  \Exp[\norm{\matR(\ww_t-\ww_*)}^2]\le
 \left(
  1-\frac{1}{9d}
  \right)^t
  \norm{\matR\ww_*}^2.
}
Since
\math{\sigma_d^2(\matR)\norm{\ww}^2
  \le\norm{\matR\ww}^2\le\sigma_1^2(\matR)\norm{\ww}^2},
\mld{
  \Exp[\norm{\ww_t-\ww_*}^2]\le
 \left(
  1-\frac{1}{9d}
  \right)^t
  \kappa(\matR)\norm{\ww_*}^2.
  \label{eq:app:w-convergence}
}
Since \math{\kappa(\matR)\approx\kappa(\matX)}
and \math{\ln(1-1/9d)\approx -1/9d},
    setting the right hand side to \math{d/n} gives
\mld{
  t\approx 9d\ln(n\kappa(\matX)\norm{\ww_*}^2/d).
}
The dependence on \math{\kappa} is now benign, in the logarithm.  

\subsubsection*{Proof of Theorem~\ref{theorem:consistent}}

The result in Theorem~\ref{theorem:consistent} is essentially the same as
the one proved in the previous section without the factor of \math{9}.
When exact leverage scores are used for sampling and
the exact preconditioner \math{\matR} is known (both coming from the
exact SVD), then
\r{eq:app:convergence} becomes
\mld{
  \Exp[\norm{\vv_t-\vv_*}^2]\le
 \left(
  1-\frac{1}{d}
  \right)
  \norm{\vv_{t-1}-\vv_*}^2.
}
and
\r{eq:app:w-convergence} becomes
\mld{
\Exp[\norm{\ww_t-\ww_*}^2]\le
 \left(
  1-\frac{1}{d}
  \right)^t
  \kappa(\matX)\norm{\ww_*}^2.
}
This then gives Theorem~\ref{theorem:consistent}.

\end{document}